\documentclass[twocolumn,10pt]{article}
\usepackage{fullpage}
\usepackage{authblk}

\usepackage{times}  %
\usepackage{helvet}  %
\usepackage{courier}  %
\usepackage[hyphens]{url}  %
\usepackage{graphicx} %
\usepackage{natbib}  %
\usepackage{caption} %
\usepackage{xcolor}         %

\usepackage[algo2e, noend, noline, ruled, linesnumbered]{algorithm2e}
\DeclareCaptionStyle{ruled}{labelfont=normalfont,labelsep=colon,strut=off} %
\usepackage{algorithm}
\usepackage{algorithmic}

\usepackage{newfloat}
\usepackage{listings}
\floatstyle{ruled}
\newfloat{listing}{tb}{lst}{}
\floatname{listing}{Listing}

\usepackage{amsmath,amsfonts}

\newcommand{\actions}{A}
\newcommand{\states}{X}

\newcommand{\argmax}{\operatorname*{argmax}}

\newcommand{\poppol}{\widetilde\pi}
\newcommand{\Poppol}{\widetilde\Pi}
\newcommand{\avgpoppol}{\overline{\widetilde\pi}}
\newcommand{\masterpol}{\widetilde\pi^*}
\newcommand{\spol}{\pi} %
\newcommand{\pol}{\bspi} %
\newcommand{\sPol}{\Pi} %
\newcommand{\Pol}{\bsPi} %
\newcommand{\avgpol}{\overline\pi}
\newcommand{\brpol}{\hat\pi}
\newcommand{\mf}{\bsmu}
\newcommand{\smf}{\mu}
\newcommand{\MF}{\bfM}
\newcommand{\sMF}{M}
\newcommand{\trainset}{{\mathcal{M}}}

\usepackage{amsthm}

\newtheorem{theorem}{Theorem}

\theoremstyle{definition}
\newtheorem{definition}{Definition}

\newcommand{\BR}{\text{BR}}

\newcommand{\bfM}{{\mathbf M}}

\newcommand{\cE}{{\mathcal E}}

\newcommand{\bbE}{{\mathbb E}}

\newcommand{\bbN}{{\mathbb N}}

\newcommand{\bbR}{{\mathbb R}}

\newcommand{\EE}{{\bbE}}

\newcommand{\NN}{{\bbN}}

\newcommand{\RR}{{\bbR}}

\newcommand{\bsmu}{{\boldsymbol{\mu}}}

\newcommand{\bspi}{{\boldsymbol{\pi}}}

\newcommand{\bsPi}{{\boldsymbol{\Pi}}}

\title{Generalization in Mean Field Games by Learning Master Policies}
\date{}

\usepackage{hyperref}
\author{%
    Sarah Perrin$^{1,\star}$,
    Mathieu Lauri{\`e}re$^{2,\star}$,
    Julien P{\'e}rolat$^3$,
    Romuald {\'E}lie$^3$,
    Matthieu Geist$^{2,\dagger}$,
    Olivier Pietquin$^{2,\dagger}$%
  }
  \renewcommand\footnotemark{}
\date{
  $^1$Univ. Lille, CNRS, Inria, Centrale Lille, UMR 9189 CRIStAL%
  \\
  $^2$Google Research, Brain Team%
  \\
  $^3$DeepMind Paris%
\thanks{$^\star$ equal contribution, $^\dagger$ equal contribution}
}

\begin{document}

\maketitle

\begin{abstract}

Mean Field Games (MFGs) can potentially scale multi-agent systems to extremely large populations of agents. Yet, most of the literature assumes a single initial distribution for the agents, which limits the practical applications of MFGs. Machine Learning has the potential to solve a wider diversity of MFG problems thanks to generalizations capacities. We study how to leverage these generalization properties to learn policies enabling a typical agent to behave optimally against any population distribution. In reference to the Master equation in MFGs, we coin the term ``Master policies'' to describe them and we prove that a single Master policy provides a Nash equilibrium, whatever the initial distribution. We propose a method to learn such Master policies. Our approach relies on three ingredients: adding the current population distribution as part of the observation, approximating Master policies with neural networks, and  training via  Reinforcement Learning and Fictitious Play. We illustrate on numerical examples not only the efficiency of the learned Master policy but also its generalization capabilities beyond the distributions used for training. 
\end{abstract}

\section{Introduction}

Although learning in games has a long history~\citep{Shannon59Chess,samuel1959some}, most of recent breakthroughs  remain limited to a small number of players, \textit{e.g.}, for chess~\citep{campbell2002deep}, Go~\citep{silver2016mastering,silver2017mastering, Silver18AlphaZero}, poker~\citep{brown2018superhuman,moravvcik2017deepstack} or even video games such as Starcraft~\citep{vinyals2019grandmaster} with a large number of agents but only a handful of competing players. The combination of game theory with multi-agent reinforcement learning has proved to be efficient~\citep{lanctot2017unified}, but learning in games involving a large number of players remains very challenging. 
Recently, Mean Field Games (MFGs), introduced concurrently by \citet{lasry2007mean} and \citet{huang2006large}, have been considered as a promising approach to address this problem. They indeed model games with an infinite number of players. Instead of taking into account interactions between individuals, MFGs model the interaction between a so-called representative agent (sampled from the population distribution) and the full population itself. As in many multi-player games, solving an MFG boils down to finding a Nash equilibrium. Intuitively, it corresponds to a situation where no player can increase their reward (or decrease their cost) by changing their strategy, given that other players keep their current behavior. MFGs are classically described with a forward-backward system of partial differential equations (PDEs) or stochastic differential equations (SDEs) and can only be solved analytically in some specific cases. When an analytical solution is not available, numerical methods such as finite differences can be called to solve the PDE system. However, these techniques do not scale well with the dimensions of the state and action spaces. Another issue with PDE methods is that they  are very sensitive to initial conditions. Especially, the policy obtained is only valid for a single initial distribution $\smf_0$ for the population over the state space. This is a strong limitation for practical applications. For example, in an evacuation or traffic-flow scenario, the solution found by a PDE solver could potentially lead to an unforeseen congestion if the agents are not initially distributed as the model expected. This could have dramatic consequences. On the other hand, solving for every possible initial distribution is of course infeasible. Following the traditional trend in the  literature, even solutions to MFGs that use most recent Machine Learning methods consider that the initial distribution is fixed and thus compute policies that are agnostic to the current population. A sensible idea to alleviate the sensitivity issue is to incorporate the population as part of the observation for the representative agent, such that it can behave optimally against the population, and not only w.r.t. its current state. Yet, using such a modification of the observation cannot be done seamlessly as the uniqueness of the initial distribution is a core assumption of existing methods, including very recent ones based on Machine Learning. 

Here we do a first crucial step in this direction using Deep Reinforcement Learning (Deep RL), which sounds particularly well fitted to overcome the aforementioned difficulty. Our core contribution is to propose the first Deep RL algorithm that calculates an optimal policy independently of the initial population distribution. 

\textbf{Main contributions. } First, we extend the basic framework of MFGs by introducing a class of \emph{population-dependent policies} enabling agents to react to any population distribution. Within this class, we identify a \emph{Master policy} and establish its connection with standard population-agnostic policies arising in MFG Nash equilibria (Thm.~\ref{thm:usual-to-master}). Second, we propose an algorithm, based on Fictitious Play and Deep RL, to learn a Master policy. We analyze a continuous time version of Fictitious Play and prove convergence at a linear rate (Thm.~\ref{thm:convergence-monotone-mf0}). Last, we provide empirical evidence that not only this method  learns the Master policy on a training set of distributions, but that the learned policy \emph{generalizes} to unseen distributions. Our approach is the first to tackle this question in the literature on MFGs. %

\section{Background and Related Works}

We consider a finite state space $\states$ and finite action space $\actions$. The set of probability distributions on $\states$ and $\actions$ are denoted by $\Delta_\states$ and $\Delta_\actions$. Let $p: \states \times \actions \times \Delta_\states \to \Delta_\states$ be a transition probability function and $r: \states  \times \actions \times \Delta_\states \to \RR$ be a reward function. Let $\gamma\in(0,1)$ be a discount parameter. In this section, we introduce the key concepts needed to explain our main contributions. Although there is no prior work tackling explicitly the question of generalization in MFG, we review along the way several related studies.

\subsection{Mean Field Games}

In the usual MFG setup~\citep{lasry2007mean,huang2006large}, a stationary policy is a function $\spol: \states \to \Delta_\actions$ and a non-stationary policy $\pol$ is an infinite sequence of stationary policies. Let $\sPol$ and $\Pol = \sPol^\NN$ be the sets of stationary and non-stationary policies respectively. Unless otherwise specified, by policy we mean a non-stationary policy. A mean-field (MF) state is a $\smf \in \Delta_\states$. It represents the state of the population at one time step. An MF flow $\mf$ is an infinite sequence of MF states. We denote by $\sMF = \Delta_\states$ and $\MF = \sMF^\NN$ the sets of MF states and MF flows. %
For $\smf \in \sMF$, $\spol \in \sPol$, let 
\begin{equation*}
    \phi(\smf,\spol): x \mapsto \sum_{x'\in\states} p(x|x',\spol(x'),\smf) \smf(x')
\end{equation*}
denote the next MF state. The MF flow starting from $\smf_0$ and controlled by $\pol \in \Pol$ is denoted by $\Phi(\smf_0,\pol) \in \MF$:
$$
    \Phi(\smf_0,\pol)_0 = \smf_0, 
    \,\, \Phi(\smf_0,\pol)_{n+1} = \phi(\Phi(\smf_0,\pol)_n, \pol_n), n \ge 0.
$$

Facing such a population behavior, an infinitesimal agent seeks to solve the following Markov Decision Process (MDP). Given an initial $\smf_0$ and a flow $\mf$, maximize:
$$
    \pol \mapsto J(\smf_0, \pol; \mf) = \EE \Big[\sum \limits_{n=0}^{+\infty} \gamma^n r(x_n, a_n, \mf_n) \Big],
$$
subject to: $x_0 \sim \smf_0$, $x_{n+1} \sim p(.|x_n, a_n, \mf_n)$, $a_n \sim \pol_n(.|x_n)$. 
Note that, at time $n$, the reward and transition depend on the current MF state $\mf_n$. So this MDP is non-stationary but since the MF flow $\mf$ is fixed and given, it is an MDP in the classical sense. In an MFG, we look for an equilibrium situation, in which the population follows a policy from which no individual player is interested in deviating. 

\begin{definition}[MFG Nash equilibrium]
\label{def:MFG-NE}
    Given $\smf_0 \in \sMF$, $(\hat \pol^{\smf_0},\hat \mf^{\smf_0}) \in \Pol \times \MF$ is an MFG Nash equilibrium (MFG-NE) consistent with $\smf_0$ if: (1) $\hat\pol^{\smf_0}$ maximizes $J(\smf_0, \cdot; \hat\mf^{\smf_0})$, and (2) $\hat\mf^{\smf_0} = \Phi(\smf_0,\hat\pol^{\smf_0})$.
\end{definition}
Being an MFG-NE amounts to say that the \emph{exploitability} $\cE(\smf_0, \hat\pol^{\smf_0})$ is $0$, where the exploitability of a policy $\pol \in \Pol$ given the initial MF state $\smf_0$ is defined as:
$$
    \cE(\smf_0, \pol) = \max \limits_{\pol'} J(\smf_0, \pol'; \Phi(\smf_0,\pol)) - J(\smf_0, \pol; \Phi(\smf_0,\pol)).
$$
 It quantifies how much a representative player can be better off by deciding to play another policy than $\pol$ when the rest of the population uses $\pol$ and the initial distribution is $\smf_0$ for both the player and the population. Similar notions are widely used in computational game theory~\citep{zinkevich2007regret,lanctot2009monte}.

In general, $\hat \pol^{\smf_0}$ is not an MFG-NE policy consistent with $\smf_0' \neq \smf_0$. Imagine for example a game in which the agents need to spread uniformly throughout a one-dimensional domain (see the experimental section). Intuitively, the movement of an agent at the center depends on where the bulk of the population is. If $\smf_0$ is concentrated on the left (resp. right) side, this agent should move towards the right (resp. left). Hence the optimal policy depends on the whole population distribution. 

Equilibria in MFG are traditionally characterized by a forward-backward system of equations~\citep{lasry2007mean,CarmonaDelarue_book_I}. Indeed, the value function of an individual player facing an MF flow $\mf$ is:
$$
    V_n(x; \mf) = \sup_{\pol \in \Pol} \EE_{x,\pol}\Big[\sum \limits_{n'=n}^{+\infty} \gamma^{n'-n} r(x_{n'}, a_{n'}, \mf_{n'}) \Big],
$$
where $x_n = x$ and $a_{n'} \sim \pol_{n'}(\cdot|x_{n'})$, $n' \ge n$. Dynamic programming yields:
$$
    V_n(x; \mf) = \sup_{\spol \in \sPol} \EE_{x,\spol}\Big[r(x_{n}, a_n, \mf_{n}) + \gamma V_{n+1}(x'; \mf)\Big],  
$$
where $x_n = x$, $a_n \sim \pi(\cdot|x)$ and $x' \sim p(\cdot|x,a,\mf_n)$. 
Taking the maximizer gives an optimal policy for a player facing $\mf$. To find an equilibrium policy, we replace $\mf$ by the equilibrium MF flow $\hat\mf$: $\hat V_n(\cdot) = V_n(\cdot; \hat\mf)$. But $\hat\mf$ is found by using the corresponding equilibrium policy. This induces a coupling between the backward equation for the representative player and the forward population dynamics.

The starting point of our Master policy approach is to notice that $V_n(\cdot; \mf)$ depends on $n$ and $\mf$ only through $(\mf_{n'})_{n' \ge n}$ hence $V_n$ depends on $n$ only through $(\mf_{n'})_{n' \ge n}$:
$$
    V_n(x; \mf)
    =
    V(x; (\mf_{n'})_{n' \ge n}) 
$$
where, for $\mf \in \MF, x \in \states$,
\begin{equation}
    \label{eq:def-V-mu}
    V(x; \mf)
    = \sup_{\spol \in \sPol} \EE_{x,\spol}\Big[r(x, a, \mf_{0}) + \gamma V(x'; (\mf_{n})_{n \ge 1})\Big],  
\end{equation}
where $a \sim \pi(\cdot|x,\mf_0)$ and $x' \sim p(\cdot|x,a,\mf_0)$. 

From here, we will express the equilibrium policy $\hat \pol_n$ as a stationary policy (independent of $n$) which takes $\hat \mf_n$ as an extra input. Replacing $n$ by $\hat \mf_n$ increases the input size but it opens new possibilities in terms of \emph{generalization in MFGs}.

\subsection{Learning in Mean Field Games}

We focus on methods involving Reinforcement Learning, or Dynamic Programming when the model is known. Learning in MFGs can also involve methods that approximate directly the forward-backward system of equations with function approximations (such as neural networks), but we will not address them here; see, \textit{e.g.}, \citep{al2018solving,carmona2021convergence}. 

In the literature, \textit{Learning} in MFGs indistinctly refers to the optimization algorithm (being most of the time the fixed point or variations of Fictitious Play), or to the subroutines involving learning that are used to compute the policy (Reinforcement Learning) or the distribution. We make here a distinction between these notions for the sake of clarity.

\paragraph{Optimization algorithm.} From a general point of view, learning algorithms for MFGs approximate two types of objects: (1) a policy for the representative agent, and (2) a distribution of the population, resulting from everyone applying the policy. This directly leads to a simple fixed-point iteration approach, in which we alternatively update the policy and the mean-field term. This approach has been used, e.g., by~\citet{guo2019learning}. However without strong hypothesis of regularity and a strict contraction property, this scheme does not converge to an MFG-NE. To stabilize the learning process and to ensure convergence in more general settings, recent papers have either added regularization~\citep{anahtarci2020q, guo2020entropy, cui2021approximately} or used Fictitious Play \citep{hadikhanloo_fictitious-play,cardaliaguet2018mean, mguni2018decentralisedli,perrin2020fictitious,delarue2021exploration}, while \citet{hadikhanloo2017learningnonatomic} and \citet{perolat2021scaling} have introduced and analyzed Online Mirror Descent.

\paragraph{Reinforcement learning subroutine.} For a given population distribution, to update the representative player's policy or value function, we can rely on RL techniques. For instance~\citet{guo2019learning,anahtarci2020q} rely on Q-learning to approximate the $Q$-function in a tabular setting, ~\citet{fu2019actorcritic} study an actor-critic method in a linear-quadratic setting, and~\citet{elie2020convergence,perrin2021flocking} solve continuous spaces problems by relying respectively on deep deterministic policy gradient
~\citep{Lillicrap2016ContinuousCW} or soft actor-critic \citep{SACDBLP:journals/corr/abs-1801-01290}. Two time-scales combined with policy gradient has been studied by~\citet{subramanian2019reinforcementstatioMFG} for stationary MFGs. Policy iterations together with sequential decomposition has been proposed by \citet{mishra2020model} while \citet{guo2020generalframework} proposes a method relying on Trust Region Policy Optimization (TRPO, \citet{schulman2015trpo}). 

\paragraph{Distribution embedding.} Another layer of complexity in MFGs is to take into consideration population distributions for large spaces or even continuous spaces. To compute MFG solutions through a PDE approach, \citet{al2018solving,carmona2021convergence} used deep neural networks to approximate the population density in high dimension. In the context of RL for MFGs, recently, \citet{perrin2021flocking} have used Normalizing Flows \citep{rezende2015variational} to approximate probability measures over continuous state space in complex environments.

\subsection{Generalization in MFGs through Master policies}

So far, learning approaches for MFGs have considered only two aspects: optimization algorithms (\textit{e.g.}, Fictitious Play or Online Mirror Descent), or model-free learning of a representative player's best response based on samples (\textit{e.g.}, Q-learning or actor-critic methods). Here, we build upon the aforementioned notions and add to this picture another dimension of learning: \emph{generalization} over population distributions. We develop an approach to learn the representative player's best response as a function of any current population distribution and not only the ones corresponding to a fixed MFG-NE. This is tightly connected with the so-called Master equation in MFGs~\citep{lionsCDF,BENSOUSSAN20151441,cardaliaguet2019master}. Introduced in the continuous setting (continuous time, continuous state and action spaces), this equation is a partial differential equation (PDE) which corresponds to the limit of systems of Hamilton-Jacobi-Bellman PDEs characterizing Nash equilibria in symmetric $N$-player games. In our discrete context, we introduce a notion of Master Bellman equation and associated Master policy, which we then aim to compute with a new learning algorithm based on Fictitious Play.

\section{Master Policies for MFGs}

We introduce the notion of Master policy and connect it to standard population-agnostic policies arising in MFG-NE.

Consider an MFG-NE $(\hat\pol^{\smf_0},\hat\mf^{\smf_0})$ consistent with some $\smf_0$. Let $\hat V(\cdot;\smf_0) = V(\cdot; \hat\mf^{\smf_0})$, \textit{i.e.},
$$
    \hat V(x; \smf_0)
    = \sup_{\spol \in \sPol} \EE_{\spol}\Big[r(x, a, \smf_0) + \gamma V(x'; (\hat\mf_{n}^{\smf_0})_{n \ge 1})\Big],  
$$
where $a \sim \pi(\cdot|x,\smf_0)$ and $x' \sim p(\cdot|x,a,\smf_0)$. By definition, $\hat\pol_0^{\smf_0}$ is a maximizer in the sup above. Moreover, in the right-hand side,
$$
    V(x'; (\hat\mf_{n}^{\smf_0})_{n \ge 1}) = \hat V(x'; \hat\mf_{1}^{\smf_0}), \qquad \hat\mf_{1}^{\smf_0} = \phi(\smf_0, \hat\pol_0^{\smf_0}).
$$
By induction, the equilibrium can be characterized as:
$$
    \begin{cases}
    \hat\pol_n^{\smf_0}
    \in \argmax_{\spol \in \sPol} \EE_{\spol}\Big[r(x, a, \hat\mf_n^{\smf_0}) + \gamma \hat V(x'; \hat\mf_{n+1}^{\smf_0})\Big]
    \\
    \hat V(x; \hat\mf_{n}^{\smf_0})
    = \EE_{\hat\pol_n^{\smf_0}}\Big[r(x, a, \hat\mf_n^{\smf_0}) + \gamma \hat V(x'; \hat\mf_{n+1}^{\smf_0})\Big]
    \\
    \hat\mf_{n+1}^{\smf_0} = \phi(\hat\mf_{n}^{\smf_0}, \hat\pol_{n}^{\smf_0}).
\end{cases}
$$
Note that $\hat\mf_{n+1}^{\smf_0}$ and $\hat\pol_{n}^{\smf_0}$ depend on each other (and also on $\smf_0$), which creates a forward-backward structure.

In the sequel, we will refer to this function $V$ as the \emph{Master value function}. Computing the value function $(x,\smf) \mapsto \hat V(x; \smf)$ would allow us to know the value of any individual state $x$ facing an MFG-NE starting from any MF state $\smf$. However, it would not allow to easily find the corresponding equilibrium policy, which still depends implicitly on the equilibrium MF flow. For this reason, we introduce the notion of population-dependent policy. The set of population-dependent policies $\poppol: \states \times \Delta_\states \to \Delta_\actions$ is denoted by $\Poppol$.

\begin{definition}
    A population-dependent $\masterpol \in \Poppol$ is a Master policy if for every $\smf_0$, $(\pol^{\smf_0,\masterpol},\mf^{\smf_0,\masterpol})$ is an MFG-NE, where: $\mf^{\smf_0,\masterpol}_{0} = \smf_0$ and for $n \ge 0$,
    \begin{equation}
    \label{eq:masterpol-Nasheq}
    \begin{cases}
        \pol^{\smf_0,\masterpol}_n(x) = \masterpol(x, \mf^{\smf_0,\masterpol}_n)
        \\
        \mf^{\smf_0,\masterpol}_{n+1} = \phi(\mf^{\smf_0,\masterpol}_{n}, \pol^{\smf_0,\masterpol}_n).
    \end{cases}
    \end{equation}
\end{definition}
\looseness=-1
A Master policy allows recovering the MFG-NE starting from any initial MF state. A core question is the existence of such a policy, which we prove in Theorem~\ref{thm:usual-to-master} below.  Hence, if there is a unique Nash equilibrium MF flow (\textit{e.g.}, thanks to monotonicity), the MF flow $\mf^{\smf_0,\masterpol}$ obtained with the Master policy $\masterpol(a|x,\mf^{\smf_0,\masterpol}_n)$ is the same as the one obtained with a best response policy $\hat\pol^{\smf_0}_n(a|x)$ starting from $\smf_0$.

\begin{theorem}
\label{thm:usual-to-master}
    Assume that, for all $\smf_0 \in \sMF$, the MFG admits an equilibrium consistent with $\smf_0$ and that the equilibrium MF flow is unique. Then there exists a Master policy $\masterpol$. 
\end{theorem}
Existence and uniqueness of the MFG-NE for a given $\smf_0$ can be proved under a mild monotonicity condition, see \textit{e.g.}~\citet{perrin2020fictitious}. Thm.~\ref{thm:usual-to-master} is proved in detail in the Appendix. We check, step by step, that the MF flow generated by $\masterpol$ and the associated population-agnostic policy as defined in~\eqref{eq:masterpol-Nasheq} form a MFG-NE. The key idea is to use dynamic programming relying on the Master value function $V$ and the uniqueness of the associated equilibrium MF flow.

\section{Algorithm}

We have demonstrated above that the Master policy is well-defined and allows to recover Nash equilibria. We now propose a method to compute such a policy.

\subsection{Fictitious Play}

We introduce an adaptation of the Fictitious Play algorithm to learn a Master policy. This extends to the case of population-dependent policies the algorithm introduced by~\citet{hadikhanloo_fictitious-play}. In the same fashion, at every iteration $k$, it alternates three steps: (1) computing a best response policy $\widetilde\pi_k$ against the current averaged MF flows $\bar\trainset_k$, (2) computing $(\mf
^{\mu_0, \widetilde \pi_k})_{\mu_0 \in \trainset}$, the MF flows induced by $\widetilde \pi_k$, and (3) updating $\bar\trainset_{k+1}$ with $(\mf
^{\mu_0, \widetilde \pi_k})_{\mu_0 \in \trainset}$. In contrast with~\citet{hadikhanloo_fictitious-play}, we learn policies that are randomized and not deterministic. 

We choose Fictitious Play rather than a simpler fixed-point approach because it is generally easier to check that an MFG model satisfies the assumptions used to prove convergence (monotonicity condition rather than contraction properties, as \textit{e.g.} in  \citet{huang2006large,guo2019learning}). %

Ideally, we would like to train the population-dependent policy on every possible distributions, but this is not feasible. %
Thus, we take a finite training set $\trainset$ of initial distributions. Each training distribution is used at each iteration of Fictitious Play. Another possibility would have been to swap these two loops, but we chose not to do this because of \emph{catastrophic forgetting} \citep{french1999catastrophicforgetting, goodfellow2015empirical}, a well-know phenomenon in cognitive science that also occurs in neural networks, describing the tendency to forget previous information when learning new information. Our proposed algorithm  is summarized in Alg.~\ref{algMEFP} and we refer to it as \emph{Master Fictitious Play}.

\begin{algorithm2e}[ht!]
\SetAlgoLined
\DontPrintSemicolon
\SetKwInOut{Input}{input}\SetKwInOut{Output}{output}
\Input{Initial $\widetilde\pi_0 \in \Poppol$, training set of initial distributions $\trainset$, number of Fictitious Play steps $K$}
\caption{Master Fictitious Play \label{algMEFP}}
Let $\bar{\pi} = \widetilde\pi_0$; let $\bar\mf^{\smf_0}_{0,n} = \smf_0$ for all $\smf_0 \in \sMF$, all $n \ge 0$\; 
Let $\bar \trainset_0 = (\bar\mf^{\smf_0}_0)_{\smf_0 \in \trainset}\; $ %

\For{$k=1, \dots, K$}{
        Train $\widetilde\pi_k$ against $\bar \trainset_k = (\bar\mf^{\smf_0}_k)_{\smf_0 \in \trainset} $, to maximize Eq.~ \eqref{eq:subroutine-optim-pb} \;
        \For{$\smf_0 \in \trainset$}{
        Compute $\mf^{\smf_0}_k$, the MF flow starting from $\smf_0$ induced by $\widetilde\pi_k$ against $\bar\mf^{\smf_0}_k$\;
        Let $\bar\mf^{\smf_0}_k = \frac{k}{k+1}\bar\mf^{\smf_0}_{k-1} + \frac{1}{k+1} \mf^{\smf_0}_k$
        }
        Update $\avgpoppol_k = \text{UNIFORM}(\widetilde\pi_0, \dots, \widetilde\pi_k)$
}
\Return{\normalfont$\avgpoppol_K = \text{UNIFORM}(\widetilde\pi_0, \dots, \widetilde\pi_K)$}
\end{algorithm2e}
Alg.~\ref{algMEFP} returns $\avgpoppol_K$, which is the uniform distribution over past policies. We use it as follows. First, let:
$\mf_{k,0}^{\smf_0} = \smf_0,$ $k=1,\dots,K,$  
    $\bar\mf_{K,0}^{\smf_0} = \frac{1}{K} \sum_{k=1}^K \mf_{k,0}^{\smf_0}$, 
and then, for $n \ge 0$,
$$
\begin{cases}
    \mf^{\smf_0}_{k,n+1} = \phi(\mf^{\smf_0}_{k,n}, \poppol_k(\cdot|\cdot,\bar\mf^{\smf_0}_{K,n})), \qquad k=1,\dots,K
    \\
    \bar\mf^{\smf_0}_{K,n+1} = \frac{1}{K} \sum_{k=1}^K \mf^{\smf_0}_{k,n+1}.
\end{cases}
$$
Note that $\avgpoppol_K$ is used in the same way for every $\smf_0$. We will show numerically that this average distribution and the associated average reward are close to the equilibrium ones.

Define the average exploitability as:
\begin{equation}
    \label{eq:avg-exploitability}
    \bar\cE_{\trainset}(\avgpoppol_K) = \EE_{\smf_0 \sim \text{UNIFORM}(\trainset)} \big[ \bar\cE(\smf_0, \avgpoppol_K) \big], 
\end{equation}
where
$$
    \bar\cE(\smf_0, \avgpoppol_K) = \max \limits_{\pol'} J(\smf_0, \pol'; \bar\mf^{\smf_0}_K) - \frac{1}{K}\sum_{k=1}^K J(\smf_0, \poppol_k; \bar\mf^{\smf_0}_K).
$$
We expect $\bar\cE(\smf_0, \avgpoppol_K) \to 0$ as $K \to +\infty$. We show that this indeed holds under suitable conditions in the idealized setting with continuous time updates, where  $\avgpoppol_k$, $k = 0,1,2,\dots$, is replaced by $\avgpoppol_t$, $ t \in [0,+\infty)$ (see Appendix for details).
\begin{theorem}\label{thm:convergence-monotone-mf0}
Assume the reward is separable and monotone, \textit{i.e.}, $r(x,a,\smf) = r_A(x,a) + r_M(x,\smf)$  and $\sum_{x \in \states}(r_M(x,\smf) - r_M(x,\smf'))(\smf-\smf')(x) < 0$ for every $\smf \neq \smf'$. Assume the transition depends only on $x$ and $a$: $p(\cdot|x,a,\smf) = p(\cdot|x,a)$. Then $\bar\cE_{\trainset}(\avgpoppol_t)  = O(1/t)$, where $\avgpoppol_t$ is the average policy at time $t$ in the continuous time version of Master Fictitious Play. 
\end{theorem}
The details of continuous time Master Fictitious Play and the proof of this result are provided in the Appendix, following the lines of~\citep{perrin2020fictitious} adapted to our setting. Studying continuous time updates instead of discrete ones enables us to use calculus, which leads to a simple proof. To the best of our knowledge, there is no rate of convergence for discrete time Fictitious Play in the context of MFG except for potential or linear-quadratic structures, see~\citep{geist2021curl} and~\citep{delarue2021exploration}.

\subsection{Deep RL to Learn a Population-dependent Policy} 

In Alg.~\ref{algMEFP}, a crucial step is to learn a population-dependent best response against the current averaged MF flows $\bar \trainset_k = (\bar\mf^{\smf_0}_k)_{\smf_0 \in \trainset}$, \textit{i.e.}, $\poppol^*_k$ maximizing
\begin{equation}
    \label{eq:subroutine-optim-pb}
    \textstyle \poppol \mapsto \frac{1}{|\trainset|}\sum_{\smf_0 \in \trainset} J(\smf_0, \poppol; \bar\mf^{\smf_0}_k).
\end{equation} 

Solving the optimization problem~\eqref{eq:subroutine-optim-pb} can be reduced to solving a standard but non-stationary MDP. Since we aim at optimizing over population-dependent policies, the corresponding $Q$-function is a function of not only an agent's state-action pair $(x,a)$ but also of the population distribution: $\widetilde Q(x,\smf, a)$. Adding the current mean field state $\smf$ to the $Q$-function allows us to recover a stationary MDP. As we know that the optimal policy is stationary, we now have a classical RL problem with state $(x, \smf)$ (instead of $x$ only), and we can use Deep RL methods such as DQN \citep{mnih2013playing} to compute $\widetilde Q_k$. The policy $\widetilde \pi_k$ can then be recovered easily by applying the $\argmax$ operator to the $Q$-function.

Various algorithms could be used, but we choose DQN 
to solve our problem because it is sample-efficient.
An algorithm detailing our adaptation of DQN to our setting is provided in Alg.~\ref{algMBR} in the Appendix. 
For the numerical results presented below, we used the default implementation of RLlib \citep{rllib}. 

The neural network representing the $Q$-function takes as inputs the state $x$ of the representative player and the current distribution $\smf$ of the population, which can simply be represented as a histogram (the proportion of agents in each state). In practice, $\smf$ is a mean-field state coming from one of the averaged MF flows $\bar\mf^{\mu_0}_k$ and is computed in steps 7 and 8 of Alg.~\ref{algMEFP} with a Monte-Carlo method, \textit{i.e.} by sampling a large number of agents that follow the last population-dependent best response $\widetilde \pi_k$ and averaging it with $\bar\mf^{\mu_0}_{k-1}$.
Then, the $Q$-function can be approximated by a feedforward fully connected neural network with these inputs. In the examples considered below, the finite state space comes from the discretization of a continuous state space in dimension~$1$ or~$2$. The aforementioned simple approximation gives good results in dimension $1$. However, in dimension $2$, the neural network did not manage to learn a good population-dependent policy in this way. This is probably because passing a histogram as a flat vector ignores the geometric structure of the problem. We thus resort to a more sophisticated representation. We first create an \emph{embedding} of the distribution by passing the histogram to a convolutional neural network (ConvNet). The output of this embedding network is then passed to a fully connected network which outputs probabilities for each action  (see~ Fig.\ref{fig:NN}). The use of a ConvNet is motivated by the fact that the state space in our examples has a clear geometric interpretation and that the population can be represented as an image. 

\begin{figure}[tbh]
    \centering
    
    \includegraphics[width=0.7\linewidth]{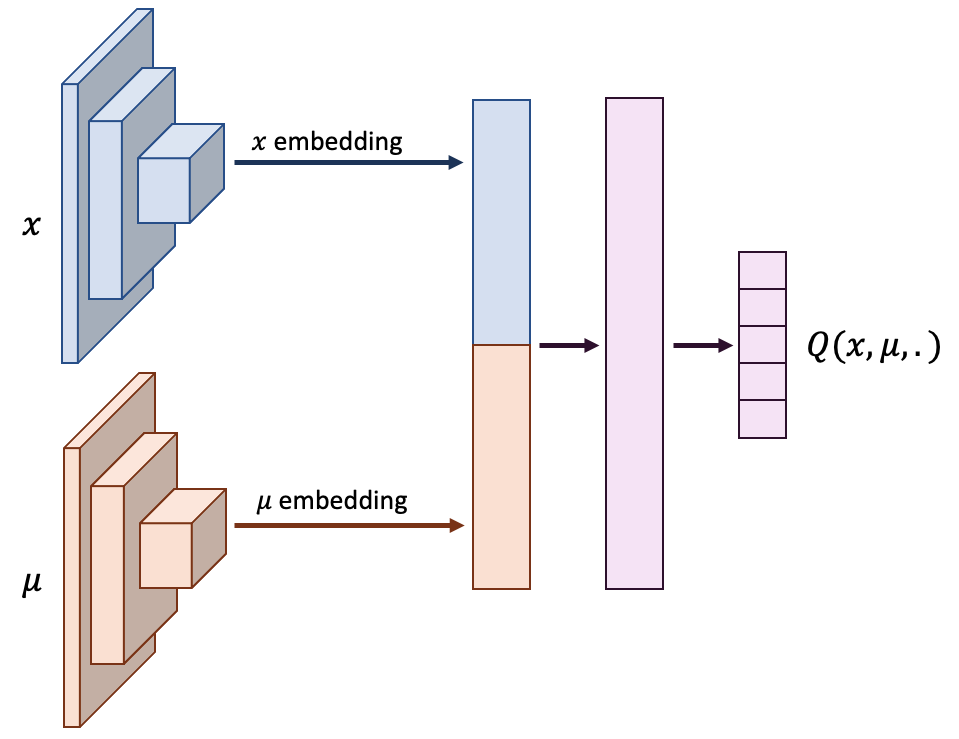}
    
    \caption{Neural network architecture of the $Q$-function for the 2D beach bar experience.
    }
    \label{fig:NN}
\end{figure}

\subsection{On the Theoretical vs. Experimental Settings}

Theoretically, we expect the algorithm Alg.~\ref{algMEFP} to converge perfectly to a Master policy. This intuition is supported by Thm.~\ref{thm:convergence-monotone-mf0} and comes from the fact that Fictitious Play has been proved to converge to  population-agnostic equilibrium policies when the initial distribution is fixed~\citep{hadikhanloo_fictitious-play,perrin2020fictitious}. However, from a practical viewpoint, here we need to make several approximations. The main one is related to the challenges of conditioning on a MF state. Even though the state space $\states$ is finite, the space of MF states $\sMF = \Delta_\states$ is infinite and of dimension equal to the number of states, which is potentially very large. This is why we need to rely on function approximation (\textit{e.g.}, by neural networks as in our implementation) to learn an optimal population-dependent policy. Furthermore, the training procedure uses only a finite (and relatively small) set of training distributions. On top of this, other more standard approximations are to be taken into account, in particular due to the use of a Deep RL subroutine.

\section{Numerical Experiments}

\subsection{Experimental Setup}

We now illustrate the efficiency and generalization capabilities of the Master policy learned with our proposed method.

\paragraph{Procedure.} To demonstrate experimentally the performance of the learned Master policy trained by Alg.~\ref{algMEFP}, we consider: several initial distributions, several benchmark policies and several metrics. For each metric, we illustrate the performance of each policy on each initial distribution. The initial distributions come from two sets: the training set $\trainset$ used in Alg.~\ref{algMEFP} and a testing set. For the benchmark policies, in the absence of a population-dependent baseline of reference (since, to the best of our knowledge, our work is the first to deal with Master policies), we focus on natural candidates that are population-agnostic. The metrics are chosen to give different perspectives: the population distribution and the policy performance in terms of reward.

\paragraph{Training set of initial distributions.}
In our experiments, we consider a training set $\trainset$ composed of Gaussian distributions such that the union of all these distributions sufficiently covers the whole state space. This ensures that the policy learns to behave on any state $x \in \states$. Furthermore, although we call ``training set'' the set of initial distributions, the policy actually sees more distributions during the training episodes. Each distribution visited could be considered as an initial distribution. Note however that it is very different from training the policy on all possible \emph{population distributions} (which is a simplex with dimension equal to the number of states, \textit{i.e.}, $32$ or $16^2=256$ in our examples). 

\paragraph{Testing set of initial distributions.}
The testing set is composed of two types of distributions. First, random distributions generated by sampling uniformly a random number in $[0,1]$ for each state independently, and then normalizing the distribution. Second, Gaussian distributions with means located between the means of the training set, and various variances (see the Appendix for a detailed representation of the training and testing sets).

\paragraph{Benchmark type 1: Specialized policies.}
For a given initial distribution $\smf_0^i$ with $i \in \{1,\dots,|\trainset|\}$, we  consider a Nash equilibrium starting from this MF state, \textit{i.e.}, a population-agnostic policy $\hat\pol^i$ and a MF flow $\hat\mf^i$ satisfying Def.~\ref{def:MFG-NE} with $\smf_0$ replaced by $\smf_0^i$. In the absence of analytical formula, we compute such an equilibrium using Fictitious Play algorithm with backward induction \citep{perrin2020fictitious}. %
We then compare our learned Master policy with each $\hat\pol^i$, either on $\smf_0^i$ or on another $\smf_0^j$. In the first case, it allows us to check the correctness of the learned Master policy, and in the second case, to show that it generalizes better than $\hat\pol^i$.

\paragraph{Benchmark type 2: Mixture-reward policy.} 
Each (population-agnostic) policy discussed above is specialized for a given $\smf_0^i$ but our Alg.~\ref{algMEFP} trains a (population-dependent) policy on various initial distributions. It is thus natural to see how the learned Master policy fares in comparison with a population-agnostic policy trained on various initial distributions. %
We thus consider another benchmark, called \emph{mixture-reward policy}, which is a population-agnostic policy trained to optimize an average reward. It is computed as the specialized policies described above but we replace the reward definition with an average over the training distributions.  
For $1 \le i \le |\trainset|$, recall $\hat\mf^{i}$ is a Nash equilibrium MF flow starting with MF state $\smf_0^i$. We consider the average reward: 
$
    \bar r_n(x,a) = \frac{1}{|\trainset|} \sum_{i=1}^{|\trainset|} r(x,a, \hat\mf^{i}_n).
$ 
The mixture-reward policy is an optimal policy for the MDP with this reward function. In our experiments, %
we compute it as for the specialized policies described above. %

\paragraph{Benchmark type 3: Unconditioned policy.} Another meaningful comparison is to use the same algorithm while removing the population input. This amounts to running Alg.~\ref{algMEFP} where, in the DQN subroutine, the $Q$-function neural network is a function of $x$ and $a$ only. So in Fig.~\ref{fig:NN}, we replace the $\smf$ input embedding by zeros. We call the resulting policy \emph{unconditioned policy} because it illustrates the performance when removing the conditioning on the MF term. This benchmark will be used to illustrate that the success of our approach is not only due to combining Deep RL with training on various $\smf_0$: conditioning the $Q$-function and the policy on the MF term plays a key role.

\paragraph{Metric 1: Wasserstein distance between MF flows.} We first measure how similar the policies are in terms of induced behavior at the scale of the population. 
Based on the Wasserstein distance $W$ between two distributions (see Appendix for details), we compute the following distance between MF flows truncated at some horizon $N_T$: %
        $$
        W_{i,j} 
        := \textstyle \frac{1}{N_T+1}\sum_{n=0}^{N_T} W(\mf^{\pol^{i},\smf_0^j}_n, \mf^{\pol^{j}, \smf_0^j}_n).
        $$
Note that $W_{i,i}=0$. The term $\mf^{\pol^{j}, \smf_0^j} = \hat\mf^j$ is the equilibrium MF flow starting from $\smf_0^j$, while $\mf^{\pol^{i},\smf_0^j}$ is the MF flow generated by starting from $\smf_0^j$ and using policy $\pol^i$.

\paragraph{Metric 2: Exploitability.} We also assess the performance of a given policy by measuring how far from being a Nash it is. To this end, we use the exploitability. 
We compute for each $i,j$: 
        $
        E_{i,j} = \cE(\smf_0^j, \hat\pol^i).
        $ 
When $i=j$, $E_{i,i} = 0$ because $(\hat\pol^i,\hat\mf^i)$ is a Nash equilibrium starting from $\smf_0^i$. 
When $i \ne j$, $E_{i,j}$ measures how far from being optimal $\hat\pol^i$ is when the population also uses $\hat\pol^i$, but both the representative player and the population start with $\smf_0^j$. If $E_{i,j}=0$, then $\hat\pol^i$ is a Nash equilibrium policy even when starting from $\smf_0^j$.

\subsection{Experiment 1: Pure exploration in 1D}
\label{sec:numerics-exploration1D}
We consider a discrete 1D environment inspired by \citet{geist2021curl}.  Transitions are deterministic, the state space is $\states=\{1, \dots, |\states| = 32\}$. The action space is $\actions = \{-1, 0, 1\}$: agents can go left, stay still or go right (as long as they stay in the state space). 
The reward penalizes the agent with the amount of people at their location, while discouraging them from moving too much: 
$
    r(x, a, \mu) = -\log(\mu(x)) - \tfrac {1}{|X|}|a|.
$ 
The training set of initial distributions $\trainset$ consists of four Gaussian distributions with the same variance but different means.  
The testing set is composed of random and Gaussian distributions with various variances (see Appendix, the distributions are represented in the same order as they are used in Fig.~\ref{fig:1D_matrices_Gaussian_training}). 
We can see that the Master policy is still performing well on these distributions, which highlights its generalization capacities. 
Note that the white diagonal is due to the fact that the Wasserstein distance and exploitability is zero for specialized baselines evaluated on their corresponding $\mu_0$. We also observe that the random policy is performing well on random distributions, and that exact solutions that have learned on a randomly generated distribution seem to perform quite well on other randomly generated distributions. We believe this is due to this specific environment, because a policy that keeps enough entropy will have a good performance.
\begin{figure}[tbh]
    \centering
    \begin{minipage}{.23\textwidth}
    \includegraphics[width=\linewidth]{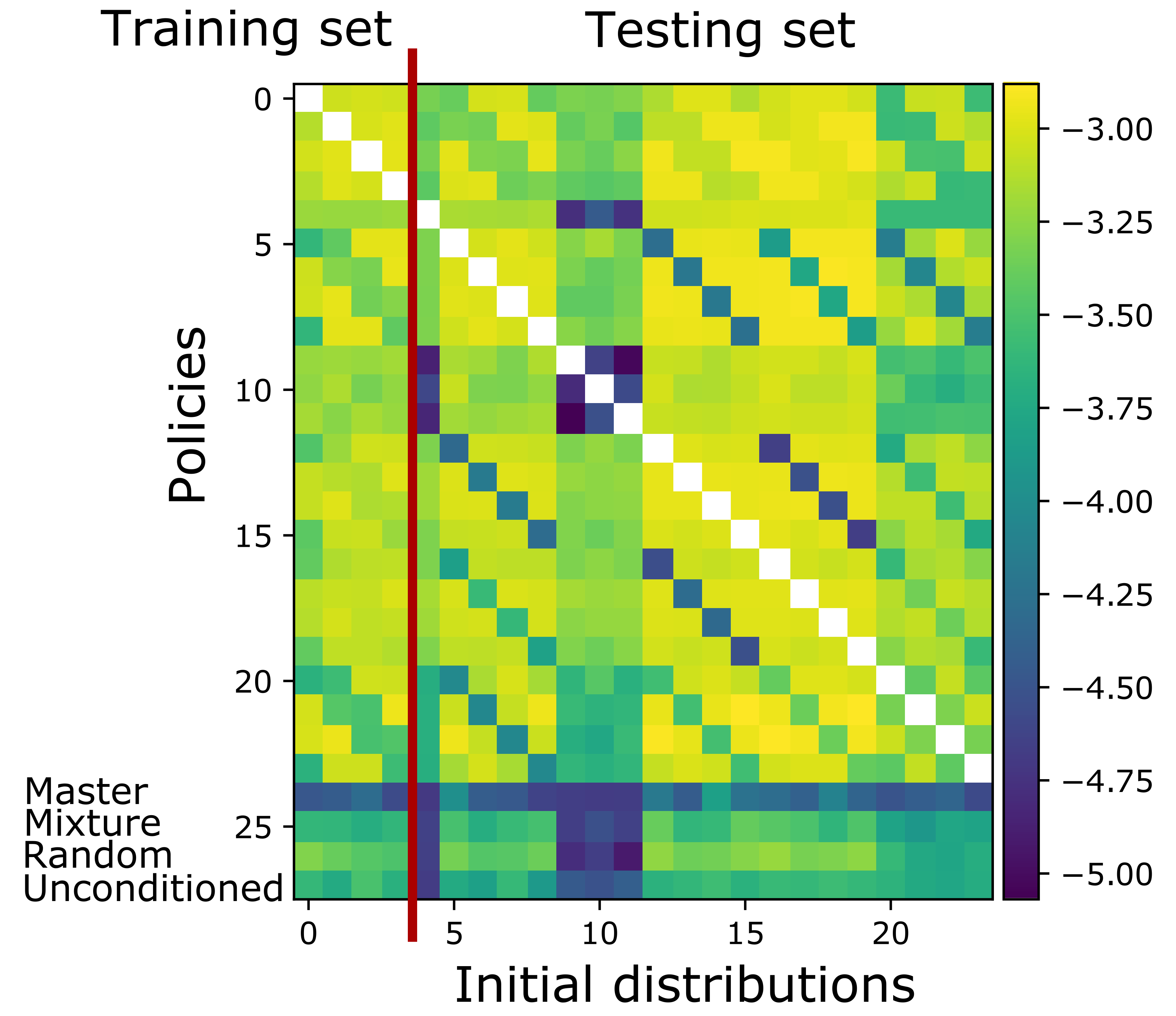}
    \end{minipage}
    \begin{minipage}{.23\textwidth}
    \includegraphics[width=\linewidth]{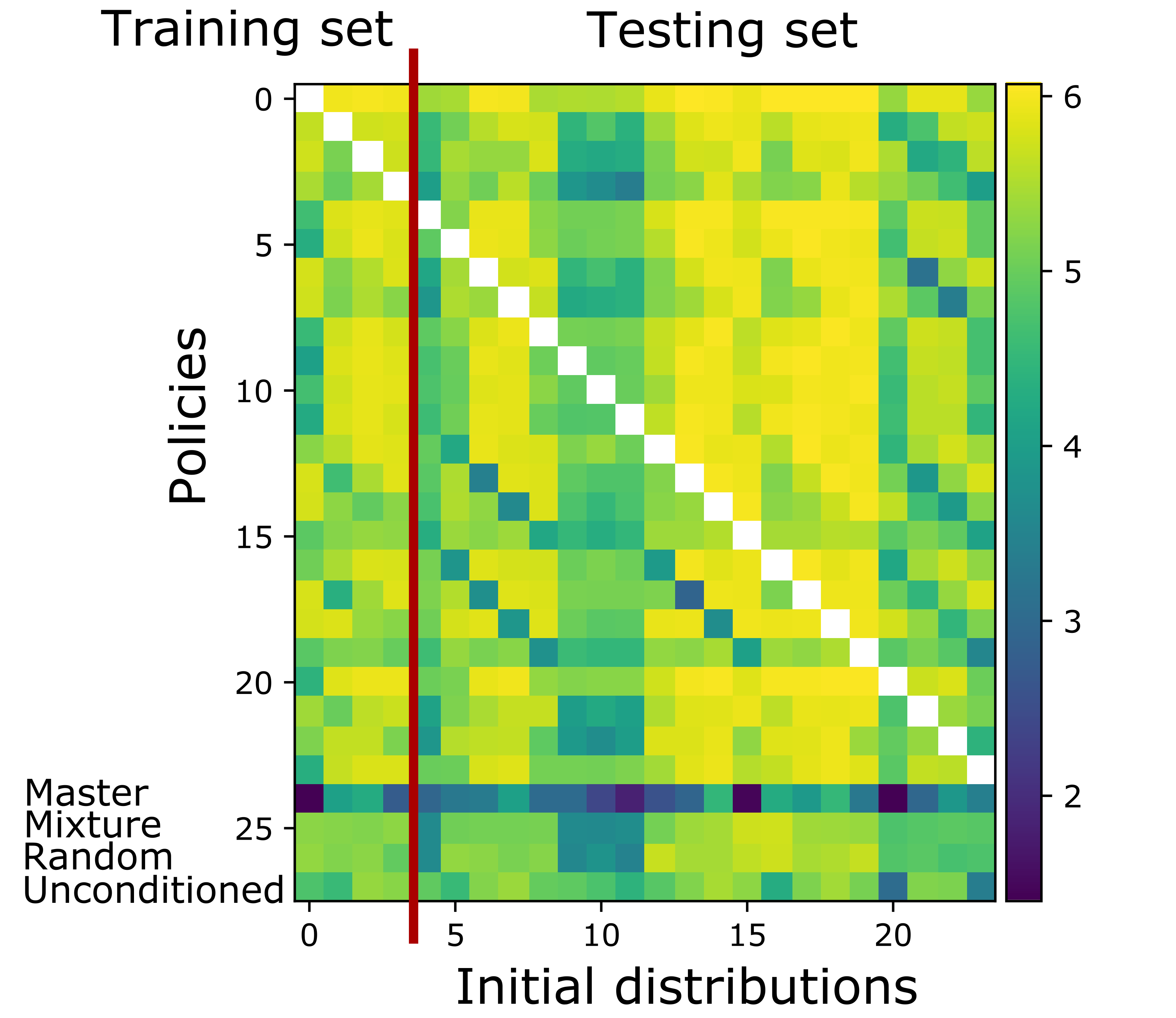}
    \end{minipage}
    
    \caption{\textbf{Exploration 1D: Performance matrices when the training set is made of Gaussian distributions.} 
    From left to right: \textbf{(a)} Log of Wasserstein distances to the exact solution (average over time steps); \textbf{(b)} Log of exploitabilities.
    }
    \label{fig:1D_matrices_Gaussian_training}
\end{figure}

\begin{figure*}[htbp]
    \centering
    \begin{minipage}{.23\textwidth}
    \includegraphics[width=\linewidth]{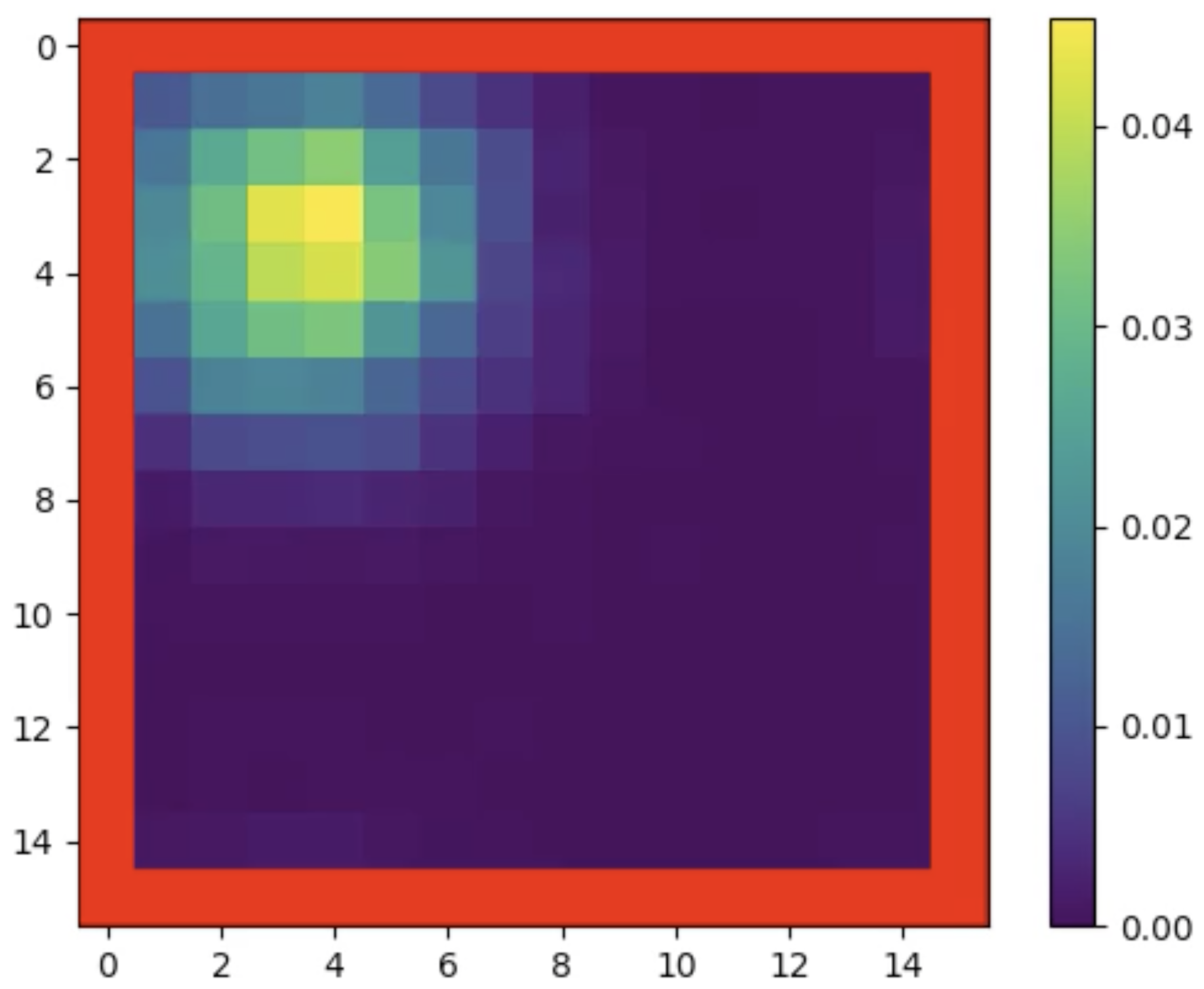}
    \end{minipage}
    \begin{minipage}{.23\textwidth}
    \includegraphics[width=\linewidth]{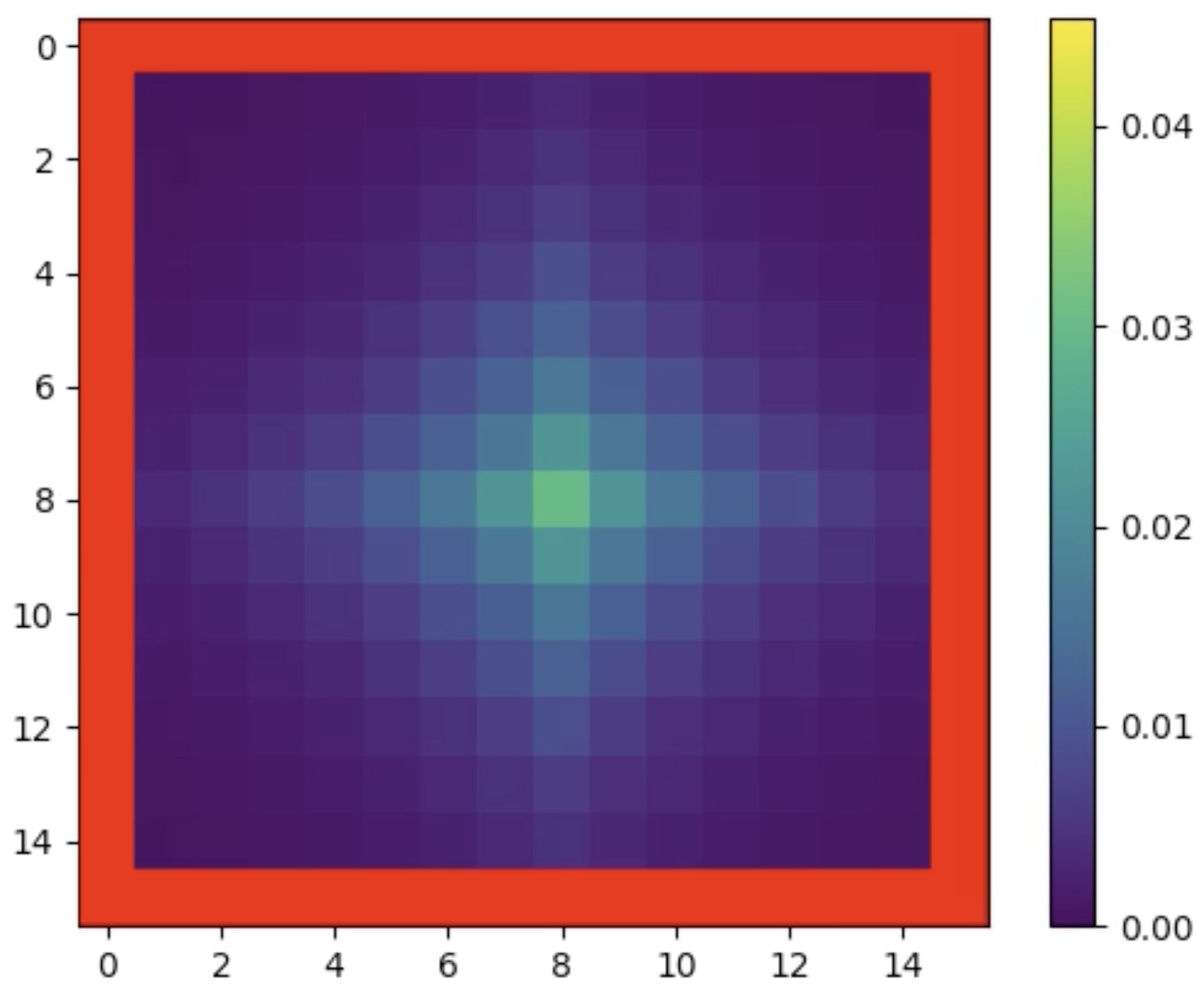}
    \end{minipage}
    \begin{minipage}{.23\textwidth}
    \includegraphics[width=\linewidth]{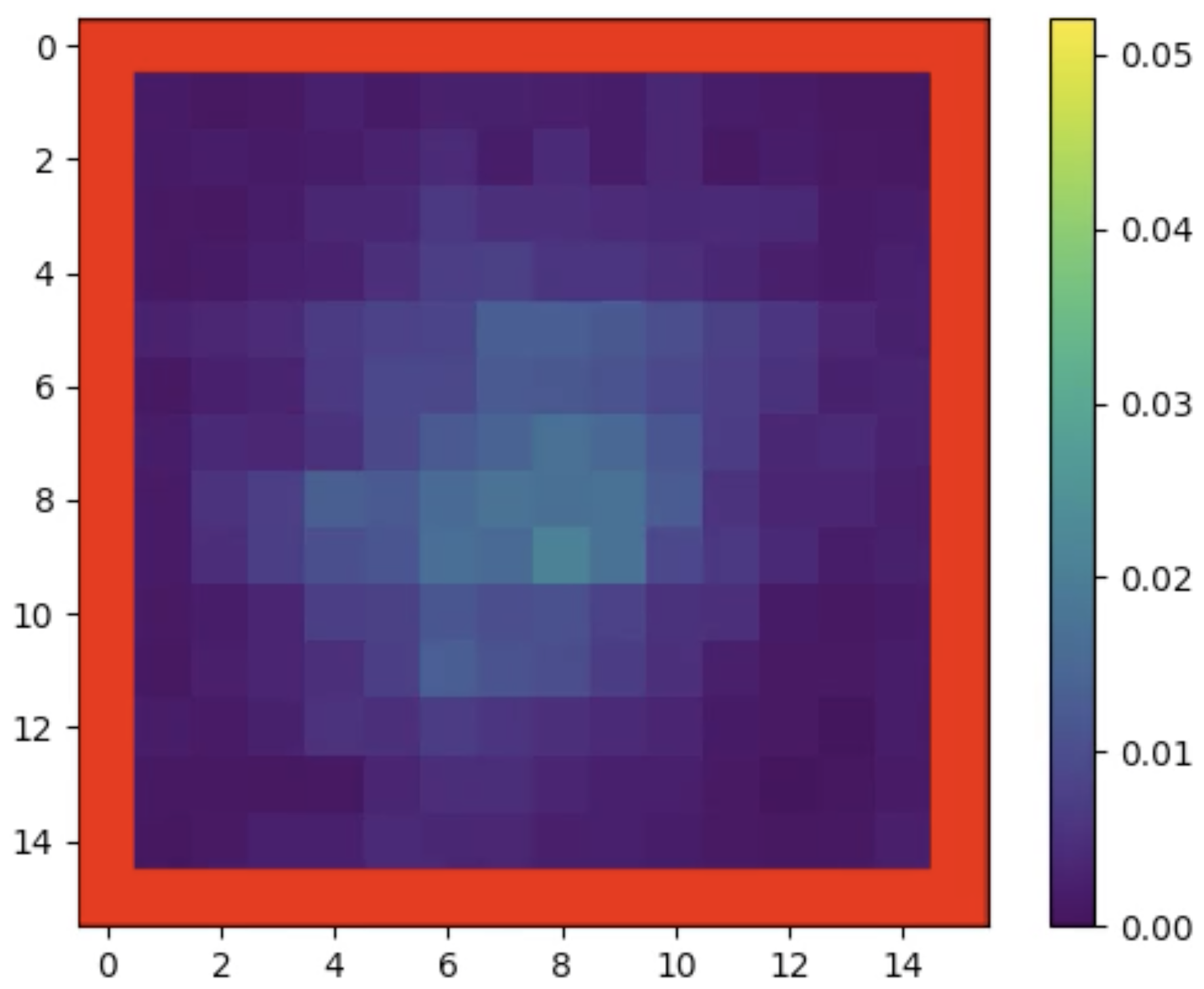}
    \end{minipage}
    \begin{minipage}{.23\textwidth}
    \includegraphics[width=\linewidth]{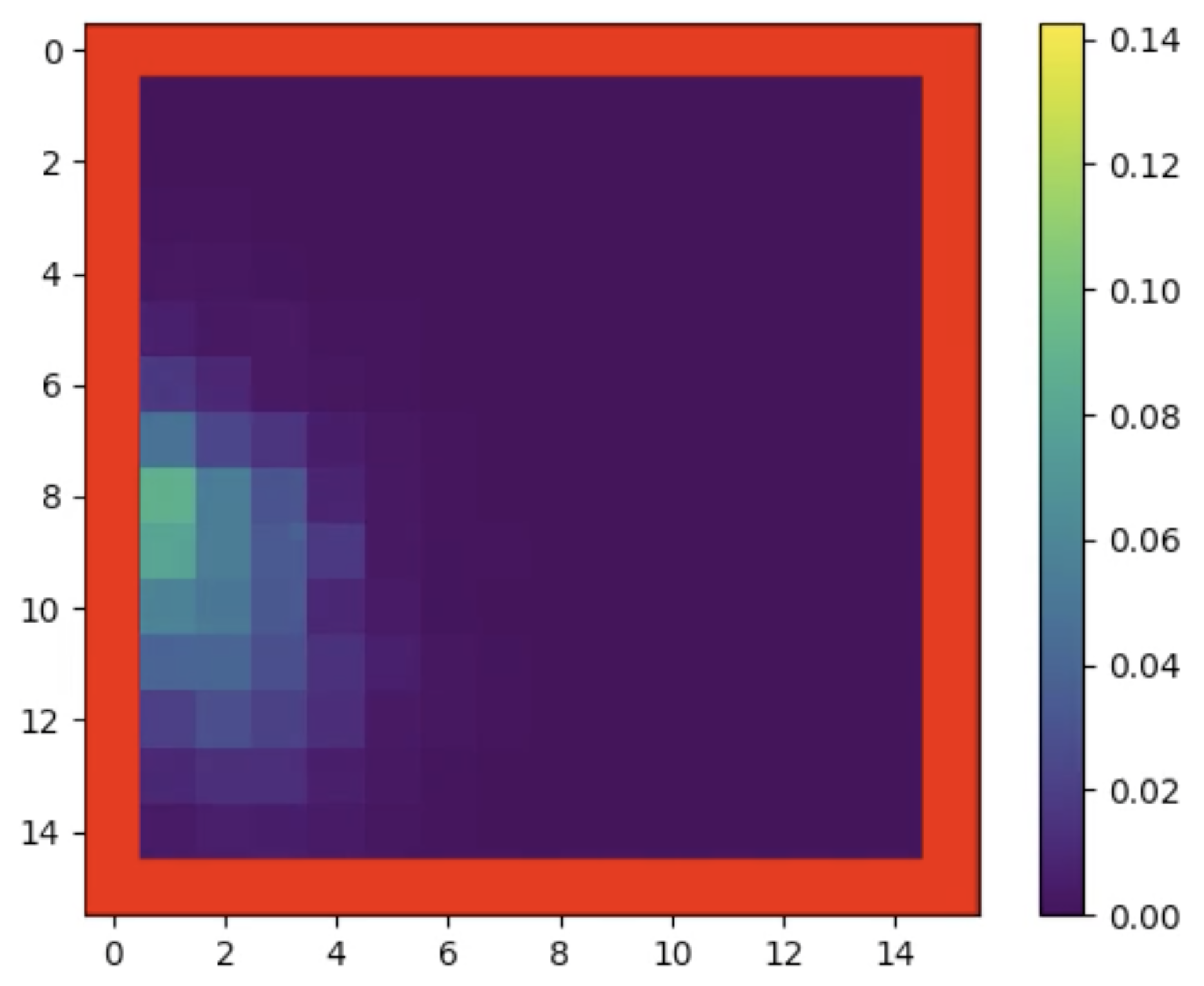}
    \end{minipage}
    
    \caption{\textbf{Beach bar 2D: Environment.} 
    From left to right: \textbf{(a)} an initial distribution $\mu_0 \in \trainset$; \textbf{(b)} MF state at equilibrium (specialized policy); \textbf{(c)} MF state at equilibrium (learned Master policy); \textbf{(d)} MF state at equilibrium (specialized policy of another initial distribution). Note that the scale is very different for the last figure.  
    }
    \label{fig:2D_different_distributions}
\end{figure*}

\subsection{Experiment 2: Beach bar in 2D}

We now consider the 2 dimensional beach bar problem, introduced by \citet{perrin2020fictitious}, to highlight that the method can scale to larger environments. The state space is a discretization of a 2-dimensional square. The agents can move by one state in the four directions: up, down, left, right, but there are walls on the boundaries. The instantaneous reward is:
$
    r(x, a, \mu) = d_{\mathrm{bar}}(x) -\log(\mu(x)) - \tfrac {1}{|X|}\|a\|_1,
$ 
where $d_{\mathrm{bar}}$ is the distance to the bar, located at the center of the domain.  Here again, the second term discourages the agent from being in a crowded state, while the last term discourages them from moving if it is not necessary. Starting from an initial distribution, we expect the agents to move towards the bar while spreading a bit to avoid suffering from congestion.%

We use the aforementioned architecture (Fig.~\ref{fig:NN}) with one fully connected network following two ConvNets: one for the agent's state, represented as a one-hot matrix, and one for the MF state, represented as a histogram. 
Having the same dimension (equal to the number $|\states|$ of states) and architecture for the position and the distribution makes it easier for the deep neural network to give an equal importance to both of these features. Deep RL is crucial to cope with the high dimensionality of the input. Here $|\states| = 16^2 = 256$. %

Fig.~\ref{fig:2D_metrics} illustrates the performance of the learned Master policy. Once again, it outperforms the specialized policies as well as the random, mixture-reward, and unconditioned policies. An illustration of the environment and of the different policies involved is available in Fig.~\ref{fig:2D_different_distributions}.
\begin{figure}[tbh]
    \centering
    \begin{minipage}{.23\textwidth}
    \includegraphics[width=\linewidth]{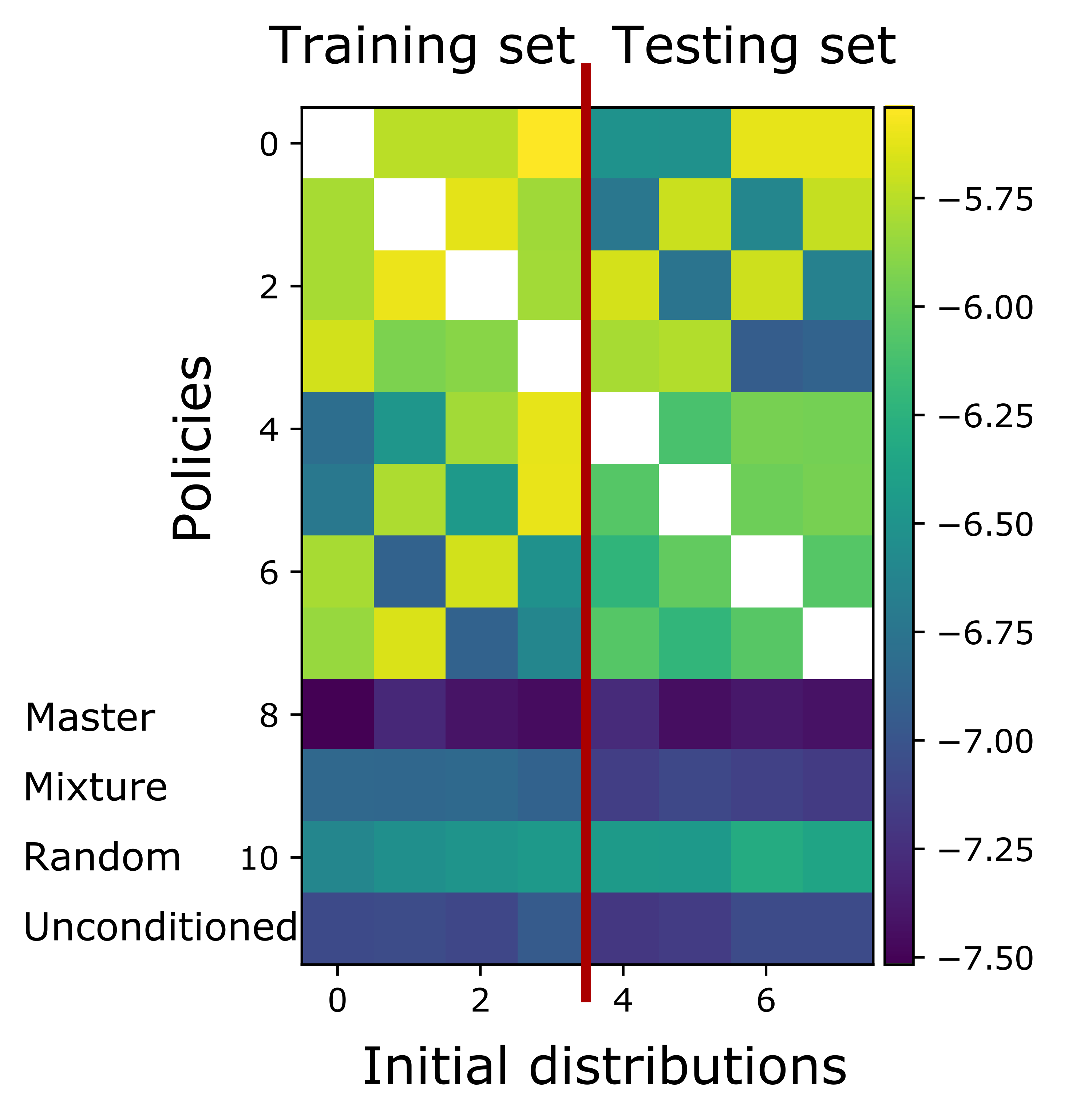}
    \end{minipage}
    \begin{minipage}{.23\textwidth}
    \includegraphics[width=\linewidth]{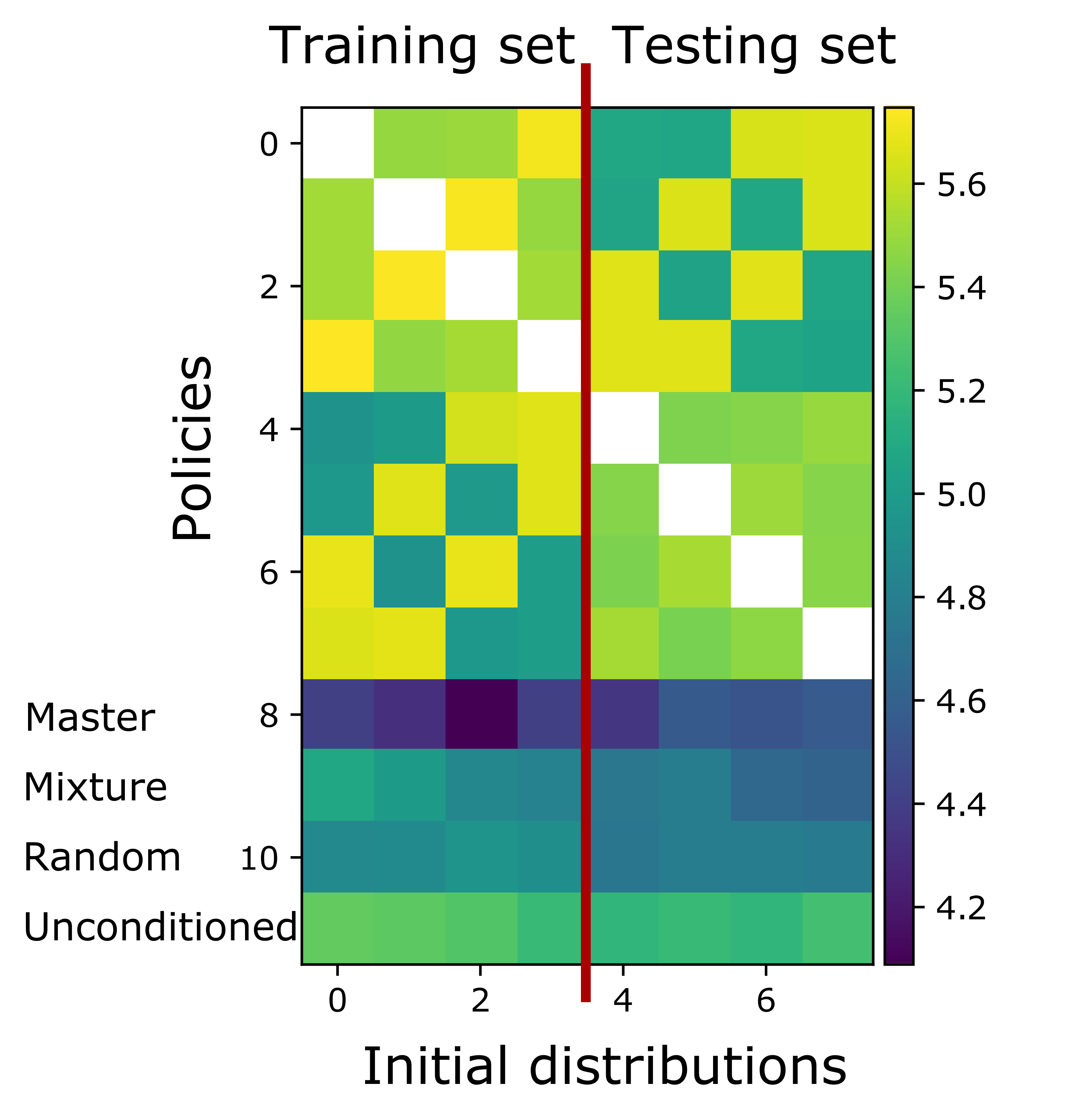}
    \end{minipage}
    
    \caption{\textbf{Beach bar 2D: Performance matrices with Gaussian distributions.} 
    From left to right: \textbf{(a)} Log of Wasserstein distances to the exact solution (average over time steps); \textbf{(b)} Log of exploitabilities.
    }
    \label{fig:2D_metrics}
\end{figure}
\section{Conclusion}

Motivated by the question of generalization in MFGs, we extended the notion of policies to let them depend explicitly on the population distribution. This allowed us to introduce the concept of Master policy, from which a representative player is able to play an optimal policy against any population distribution, as we proved in Thm.~\ref{thm:usual-to-master}. We then proved that a continuous time adaptation of Fictitious Play can approximate the Master policy at a linear rate (Thm.~\ref{thm:convergence-monotone-mf0}). However, implementing this method is not straightforward because policies and value functions are now functions of the population distribution and, hence, out of reach for traditional computational methods. We thus proposed a Deep RL-based algorithm to compute an approximate Master policy. Although this algorithm trains the Master policy using a small training set of distributions, we demonstrated numerically that the learned policy is competitive on a variety of unknown distributions. In other words, for the first time in the RL for MFG literature, our approach allows the agents to generalize and react to many population distributions. 
This is in stark contrast with the existing literature, which focuses on learning population-agnostic policies, see \textit{e.g.}~\citep{guo2019learning,anahtarci2020q,fu2019actorcritic,elie2020convergence,perrin2021flocking}. To the best of our knowledge, the only work considering policies that depend on the population is \citep{mishra2020model}, but their approach relies on solving a fixed point at each time step for every possible distribution, which is infeasible except for very small state space. %
A deep learning approach to solve finite-state Master equations has been proposed in~\citep[Section 7.2]{lauriere2021numericalAMS}, but it is based on a PDE viewpoint and hence on the full knowledge of the model. Some numerical aspects of Bellman equations involving population distributions have also been discussed in the context of mean field control, with knowledge of the model~\citep{germain2021deepsets} or without~\citep{carmona2019modelfreemfrl,gu2020meanfieldcontrols,mottepham2019mean}. However, these approaches deal only with optimal control problems and not Nash equilibria as in our case. Furthermore, none of these works treated the question of generalization in MFGs.

\looseness=-1
Our approach opens many directions for future work. First, the algorithm we proposed should be seen as a proof of concept and we plan to investigate other methods, such as Online Mirror Descent~\citep{hadikhanloo2017learningnonatomic,perolat2021scaling}. For high-dimensional examples, the question of distribution embedding deserves a special attention. Second, the generalization capabilities of the learned Master policy offers many new possibilities for applications. Last, the theoretical properties (such as the approximation and generalization theory) are also left for future work. An interesting question, from the point of view of learning is choosing the training set so as to optimize generalization capabilities of the learned Master policy.

\bibliographystyle{plainnat}
\bibliography{mfgrlbib}

\onecolumn

\appendix

\section{Notations used in the text}

The main notations used in the text are summarized in the following table. Please note that $\pol$, $\avgpol$ and $\brpol$ are population-agnostic policies, while $\poppol$, $\avgpoppol$ and $\masterpol$ are population-dependent policies.
\begin{center}
\begin{tabular}{| l | l |}
\hline
 Policy & $\pol \in \Pol$ \\ 
\hline
 Average policy & $\avgpol \in \Pol$ \\  
\hline
 Equilibrium policy & $\brpol \in \Pol$ \\
\hline
 Population-dependent policy & $\poppol \in \Poppol$\\
\hline
 Average population-dependent policy & $\avgpoppol \in \Poppol$ \\
\hline
 Master policy & $\masterpol \in \Poppol$ \\
\hline
 Mean field state & $\smf \in \sMF$ \\
\hline
 Mean field flow & $\mf \in \MF$ \\
\hline
 Training set of initial distributions & $\trainset \subset \sMF$ \\
\hline
\end{tabular}
\end{center}

\section{Details on the Experiments}

\paragraph{Wasserstein distance.} The Wasserstein distance $W$ (or earth mover's distance) measures the minimum cost of turning one distribution into another: for $\mu,\mu' \in \sMF = \Delta_\states$, 
$$ 
    W(\mu,\mu')
    = \inf_{\nu \in \Gamma(\mu,\mu')} \sum_{(x,x') \in \states \times \states} d(x, x') \nu(x,x'),
$$
where $\Gamma(\mu,\mu')$ is the set of probability distributions on $\states\times\states$ with marginals $\mu$ and $\mu'$. %
This notion is well defined if the state space has a natural notion of distance $d$, which is the case in our numerical examples because they come from the discretization of 1D or 2D Euclidean domains.

\paragraph{Initial distributions. }
We provide here a representation of the initial distributions used in the experiments.

For the pure exploration model in 1D, the training and testing sets are represented in Fig.~\ref{fig:1D_training_set} and Fig.~\ref{fig:1D_testing_set} respectively.

\begin{figure*}[htbp]
    \centering
    \begin{minipage}{.23\textwidth}
    \includegraphics[width=\linewidth]{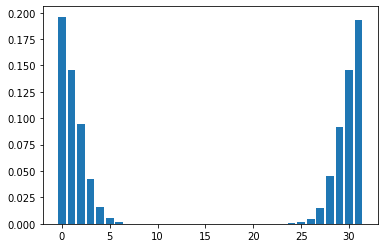}
    \end{minipage}
    \begin{minipage}{.23\textwidth}
    \includegraphics[width=\linewidth]{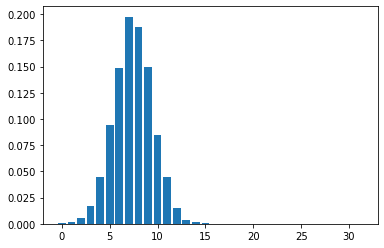}
    \end{minipage}
    \begin{minipage}{.23\textwidth}
    \includegraphics[width=\linewidth]{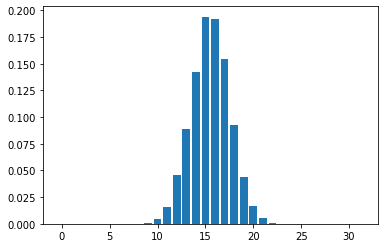}
    \end{minipage}
    \begin{minipage}{.23\textwidth}
    \includegraphics[width=\linewidth]{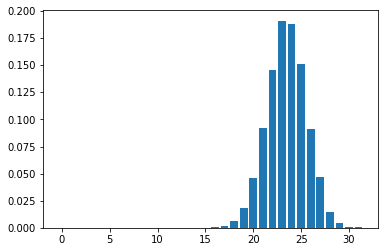}
    \end{minipage}
    
    \caption{Pure exploration 1D: Training set
    }
    \label{fig:1D_training_set}
\end{figure*}

\begin{figure*}[htbp]
    \centering
    \begin{minipage}{.23\textwidth}
    \includegraphics[width=\linewidth]{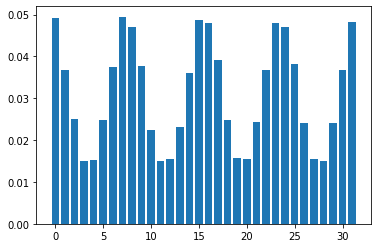}
    \end{minipage}
    \begin{minipage}{.23\textwidth}
    \includegraphics[width=\linewidth]{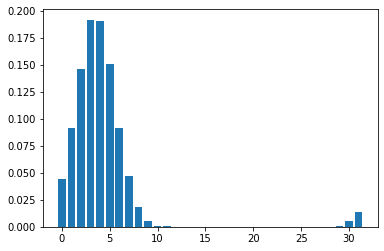}
    \end{minipage}
    \begin{minipage}{.23\textwidth}
    \includegraphics[width=\linewidth]{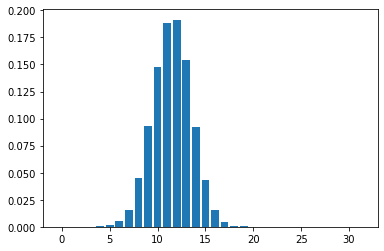}
    \end{minipage}
    \begin{minipage}{.23\textwidth}
    \includegraphics[width=\linewidth]{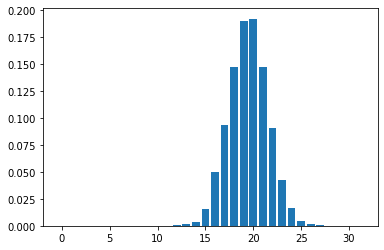}
    \end{minipage}

    \begin{minipage}{.23\textwidth}
    \includegraphics[width=\linewidth]{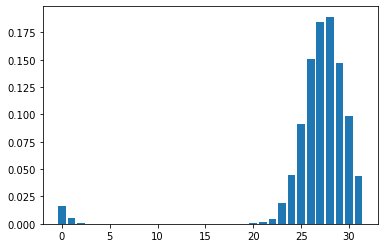}
    \end{minipage}
    \begin{minipage}{.23\textwidth}
    \includegraphics[width=\linewidth]{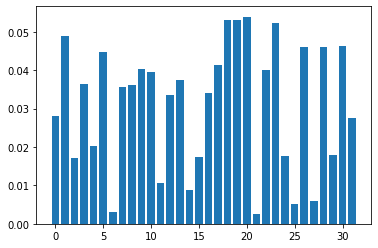}
    \end{minipage}
    \begin{minipage}{.23\textwidth}
    \includegraphics[width=\linewidth]{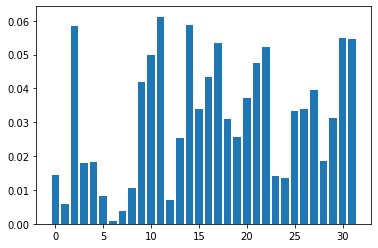}
    \end{minipage}
    \begin{minipage}{.23\textwidth}
    \includegraphics[width=\linewidth]{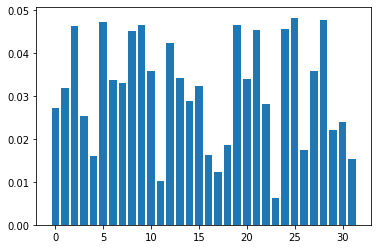}
    \end{minipage}

    \begin{minipage}{.23\textwidth}
    \includegraphics[width=\linewidth]{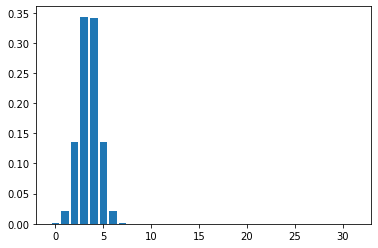}
    \end{minipage}
    \begin{minipage}{.23\textwidth}
    \includegraphics[width=\linewidth]{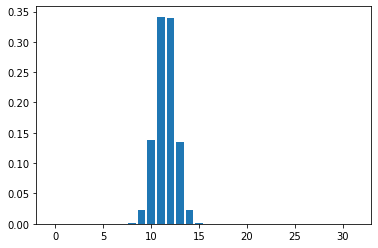}
    \end{minipage}
    \begin{minipage}{.23\textwidth}
    \includegraphics[width=\linewidth]{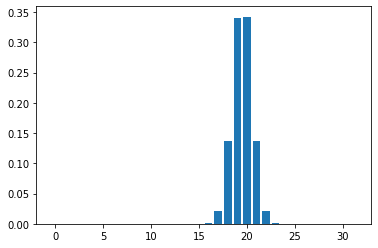}
    \end{minipage}
    \begin{minipage}{.23\textwidth}
    \includegraphics[width=\linewidth]{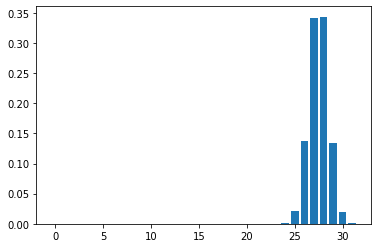}
    \end{minipage}

    \begin{minipage}{.23\textwidth}
    \includegraphics[width=\linewidth]{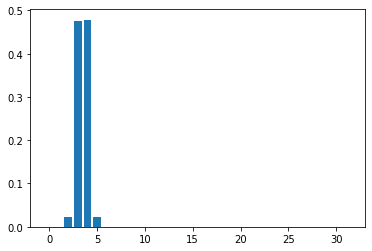}
    \end{minipage}
    \begin{minipage}{.23\textwidth}
    \includegraphics[width=\linewidth]{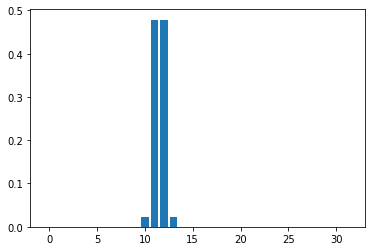}
    \end{minipage}
    \begin{minipage}{.23\textwidth}
    \includegraphics[width=\linewidth]{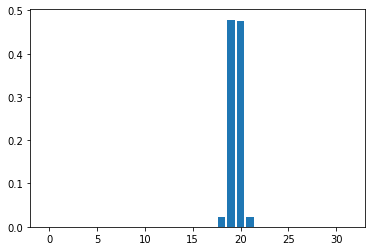}
    \end{minipage}
    \begin{minipage}{.23\textwidth}
    \includegraphics[width=\linewidth]{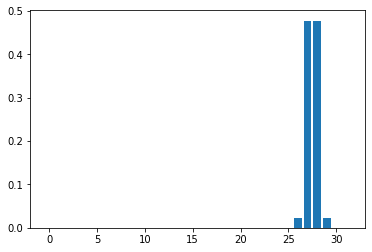}
    \end{minipage}

    \begin{minipage}{.23\textwidth}
    \includegraphics[width=\linewidth]{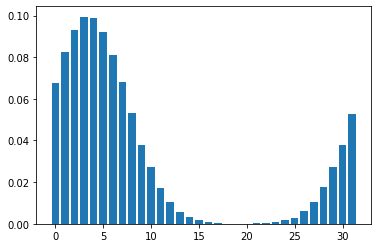}
    \end{minipage}
    \begin{minipage}{.23\textwidth}
    \includegraphics[width=\linewidth]{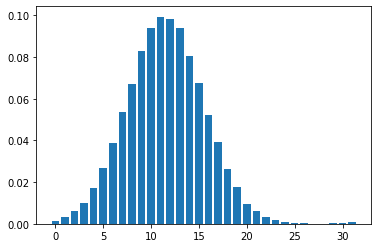}
    \end{minipage}
    \begin{minipage}{.23\textwidth}
    \includegraphics[width=\linewidth]{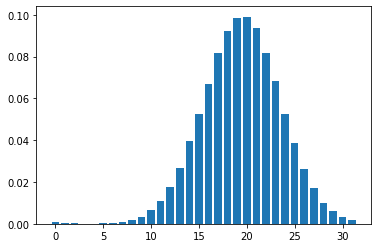}
    \end{minipage}
    \begin{minipage}{.23\textwidth}
    \includegraphics[width=\linewidth]{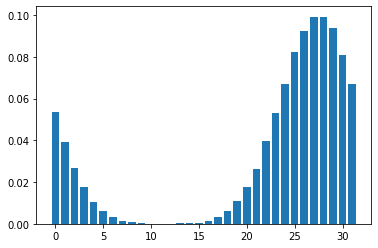}
    \end{minipage}
    
    \caption{Pure exploration 1D: Testing set
    }
    \label{fig:1D_testing_set}
\end{figure*}

For the beach bar model in 2D, the training and testing sets are represented in Fig.~\ref{fig:2D_training_set} and Fig.~\ref{fig:2D_testing_set} respectively.

\begin{figure*}[htbp]
    \centering
    \begin{minipage}{.23\textwidth}
    \includegraphics[width=\linewidth]{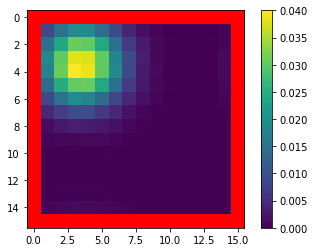}
    \end{minipage}
    \begin{minipage}{.23\textwidth}
    \includegraphics[width=\linewidth]{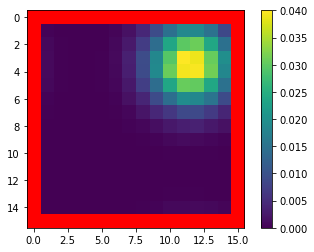}
    \end{minipage}
    \begin{minipage}{.23\textwidth}
    \includegraphics[width=\linewidth]{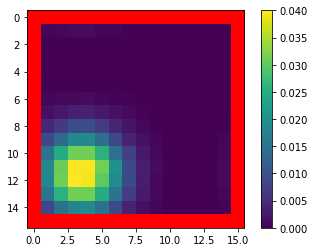}
    \end{minipage}
    \begin{minipage}{.23\textwidth}
    \includegraphics[width=\linewidth]{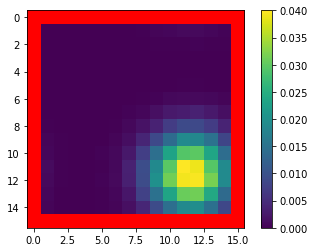}
    \end{minipage}
    
    \caption{Beach bar 2D: Training set
    }
    \label{fig:2D_training_set}
\end{figure*}

\begin{figure*}[htbp]
    \centering
    \begin{minipage}{.23\textwidth}
    \includegraphics[width=\linewidth]{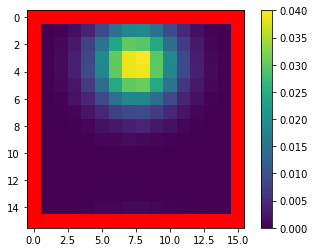}
    \end{minipage}
    \begin{minipage}{.23\textwidth}
    \includegraphics[width=\linewidth]{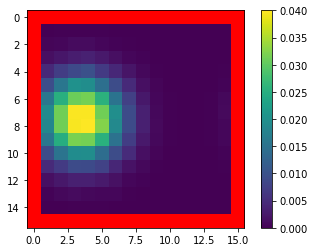}
    \end{minipage}
    \begin{minipage}{.23\textwidth}
    \includegraphics[width=\linewidth]{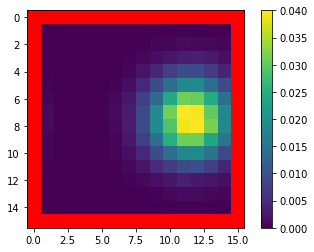}
    \end{minipage}
    \begin{minipage}{.23\textwidth}
    \includegraphics[width=\linewidth]{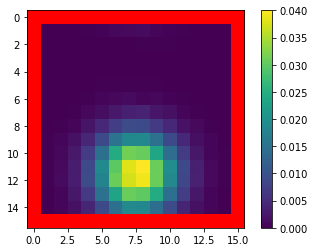}
    \end{minipage}
    
    \caption{Beach bar 2D: Testing set
    }
    \label{fig:2D_testing_set}
\end{figure*}

\section{Learning a population-dependent policy with Deep RL}\label{sec:DQN-algo}

\begin{algorithm2e}[ht!]
\SetAlgoLined
\DontPrintSemicolon
\SetKwInOut{Input}{input}\SetKwInOut{Output}{output}
\Input{Initial weights $\theta_k$ and $\theta_k'$ for network $\widetilde Q_{\theta_k}$ and target network $\widetilde Q_{\theta_k'}$; training set $\trainset$ of initial distributions; set $\bar \trainset_k$ of average MF flows; 
number of episodes $N_{\mathrm{episodes}}$; number of inner steps $N$; horizon $N_T$ for estimation; number of steps $C$ between synchronization of the two networks; parameter $\epsilon \in [0,1]$ for exploration}
\caption{DQN for a population-dependent Best Response\label{algMBR}}
Initialize weights $\theta_k$ of network $\widetilde Q_{\theta_k}$ and weights $\theta_k'$ of target network $\widetilde Q'_{\theta_k'}$\;
Initialize replay memory $B$ \;
    
    \For{$e=1, \dots, N_{\mathrm{episodes}}$}{
        Sample initial $\smf_0 \in \trainset$ and get the associated $\bar\mf^{\mu_0}_k$ from $\bar \trainset_k$\;
        Sample $x_0 \sim \smf_0$\;
        \For{$n=0, \dots, N-1$}{
            With probability $\epsilon$ select random action $a_n$, otherwise select $a_n \in \argmax_a \widetilde Q'_k(a|x_n, \bar\mf^{\mu_0}_{k,n})$\;
            Execute action $a_n$, observe reward $r_n$ and state $x_{n+1}$\;
            
            Add the transition $(x_n, a_n, \bar\mf^{\mu_0}_{k,n}, r_n, \bar\mf^{\mu_0}_{k,n+1})$ to $B$\;
            }
            Sample a random minibatch of $N_T$ transitions $\{(x_n, a_n, \smf_n, r_n, \smf_{n+1}), n=1,\dots,N_T\}$ from $B$\;
            Let  $v_n = r_n + \gamma \max_{a'} \widetilde Q_{\theta_k}(x_{n+1},\smf_{n+1}, a')$ for $n=1,\dots,N_T$\;
            Update $\theta_k$ by performing a gradient step in the direction of minimizing w.r.t. $\theta$: 
                $$
                \frac{1}{N_T}\sum_{n=1}^{N_T}\Big|v_n - \widetilde Q_{\theta}(x_n,\smf_{n}, a_n)\Big|^2
                $$\;
            Every $C$ steps, copy weights $\theta_k$ of $\widetilde Q_{\theta_k}$ to the weights $\theta_k'$ of $\widetilde Q'_{\theta_k'}$\;
        
    }
\Return{\normalfont $\widetilde Q_{\theta_k} $}
\end{algorithm2e}

Recall that in line~4 of Alg.~\ref{algMEFP}, we want solve an MDP which is stationary because we have put the distribution $\smf$ as an input together with the agent's state $x$. To this end, we use DQN and, as described in Alg.~\ref{algMBR}, we use a finite horizon approximation $N_T$. This approximation is common in the literature and is not problematic as we set the horizon high enough so that the stationary population distribution can be (approximately) reached.

\section{Proof of Theorem~\ref{thm:usual-to-master}}

\begin{proof}[Proof of Theorem~\ref{thm:usual-to-master}]

    By assumption, for every $\smf_0$, there is a unique equilibrium MF flow $\hat\mf^{\smf_0}$. We also consider an associated equilibrium (population-agnostic) policy $\hat\pol^{\smf_0}$ (if there are multiple choices of such policies, we take one of them). The superscript is used to stress the dependence on the initial MF state. Let us define the following population dependent policy:
    \begin{equation}
    \label{eq:def-masterpol-hatpol}
        \poppol(x,\smf_0) := \hat\pol^{\smf_0}_0(x).
    \end{equation}
    
    We prove that any population-dependent policy defined in the above way is a master policy, \textit{i.e.}, for each $\smf_0$ it gives an equilibrium policy not only at initial time but at all time steps.

    Fix $\smf_0$. Let $\tilde\mf^{\smf_0}$ and $\tilde\pol^{\smf_0}$ be the MF flow and the population-agnostic policy induced by using $\poppol$ starting from $\smf_0$, \textit{i.e.}, for $n \ge 0$,
    \begin{equation}
    \label{eq:tilde-pol-poppol-evol}
        \tilde\pol^{\smf_0}_n(x) = \poppol(x,\tilde\mf^{\smf_0}_n), 
        \qquad
        \tilde\mf^{\smf_0}_{n+1} = \phi(\tilde\mf^{\smf_0}_n, \tilde\pol^{\smf_0}_n). 
    \end{equation}
    We check that it is a Nash equilibrium starting with $\smf_0$. The second condition in Def.~\ref{def:MFG-NE} is automatically satisfied by definition of $\tilde\mf^{\smf_0}_{n+1}$, see~\eqref{eq:tilde-pol-poppol-evol}. For the optimality condition, we proceed by induction to show that for every $n \ge 0$, $\tilde\mf^{\smf_0}_{n} = \hat\mf^{\smf_0}_{n}$, which is the unique equilibrium MF flow starting from $\smf_0$. Note first that, by~\eqref{eq:def-masterpol-hatpol} and dynamic programming,
    $$
        \tilde\pol^{\smf_0}_0(x) = \poppol(x,\smf_0)
        = \hat\pol^{\smf_0}_0(x)
        \in \argmax_{\spol \in \sPol} \mathbb{E} \Big[r(x, a, \smf_0) + \gamma \hat V(x_1; \hat\mf^{\smf_0}_1)\; \Big| \;  x_{1} \sim p(.|x, a, \smf_0), \; a \sim \spol(.|x) \Big],
    $$
    where $\hat V$ is the stationary and population-dependent value function for a representative agent facing a population playing according to a Nash equilibrium starting from a given distribution.  Moreover, 
    $$
        \tilde\mf^{\smf_0}_{1}
        = \phi(\tilde\mf^{\smf_0}_0, \tilde\pol^{\smf_0}_0)
        = \phi(\smf_0, \poppol(\cdot,\smf_0))
        = \phi(\smf_0, \hat\pol^{\smf_0}_0)
        = \hat\mf^{\smf_0}_{1},
    $$
    where we used~\eqref{eq:tilde-pol-poppol-evol} for the first and second equalities, and~\eqref{eq:def-masterpol-hatpol} for the third equality. The last equality holds because $(\hat\mf^{\smf_0},\hat\pol^{\smf_0})$ is an MFG Nash equilibrium consistent with $\smf_0$.  So:
     $$
        \tilde\pol^{\smf_0}_0(x) 
        \in \argmax_{\spol \in \sPol} \mathbb{E} \Big[r(x, a, \smf_0) + \gamma \hat V(x_1; \tilde\mf^{\smf_0}_1)\; \Big| \;  x_{1} \sim p(.|x, a, \smf_0), \; a \sim \spol(.|x) \Big],
    $$

    At time $n \ge 1$, for the sake of induction, assume $\tilde\mf^{\smf_0}_{i} = \hat\mf^{\smf_0}_{i}$ for all $i \le n$. Then
    \begin{equation}
    \label{eq:poln-hatpol0}
        \tilde\pol^{\smf_0}_n(x)
        = \poppol(x, \tilde\mf^{\smf_0}_n)
        =\hat\pol^{\tilde\mf^{\smf_0}_n}_0(x)
        \in \argmax_{\spol \in \sPol} \mathbb{E} \Big[r(x, a, \tilde\mf^{\smf_0}_n) + \gamma \hat V(x_1; \hat\mf^{\tilde\mf^{\smf_0}_{n}}_1)\; \Big| \;  x_{1} \sim p(.|x, a, \tilde\mf^{\smf_0}_n), \; a \sim \spol(.|x) \Big].
    \end{equation}
    Moreover,
    $$
        \tilde\mf^{\smf_0}_{n+1}
        = \phi(\tilde\mf^{\smf_0}_{n}, \tilde\pol^{\smf_0}_{n})
        = \phi(\hat\mf^{\smf_0}_{n}, \hat\pol^{\hat\mf^{\smf_0}_{n}}_0)
        \underbrace{=}_{(\star)} \phi(\hat\mf^{\smf_0}_{n}, \hat\pol^{\smf_0}_{n})
        = \hat\mf^{\smf_0}_{n+1},
    $$
    where the first equality is by~\eqref{eq:tilde-pol-poppol-evol} and the second equality is by the induction hypothesis and~\eqref{eq:poln-hatpol0}. 
    Equality $(\star)$ means that the population distributions generated at the next time step by $\hat\pol^{\hat\mf^{\smf_0}_n}_0$ and $\hat\pol^{\smf_0}_{n}$ when starting from $\hat\mf^{\smf_0}_{n}$ are the same (although these two policies could be different). This is because both of them are best responses to this population distribution and because we assumed uniqueness of the equilibrium MF flow. Indeed, by definition, $\hat\pol^{\hat\mf^{\smf_0}_n}_0$ is the initial step of a policy which is part of an MFG Nash equilibrium consistent with $\hat\mf^{\smf_0}_n$. Furthermore, $(\hat\pol^{\smf_0}_{n}, \hat\mf^{\smf_0}_{n})_{n \ge 0}$ is an MFG Nash equilibrium consistent with $\smf_0$ and, as a consequence, for any $n_0 \ge 0$, $(\hat\pol^{\smf_0}_{n}, \hat\mf^{\smf_0}_{n})_{n \ge n_0}$ is an MFG Nash equilibrium consistent with $\hat\mf^{\smf_0}_{n_0}$.

   We conclude that $(\star)$ holds by using the fact that we assumed uniqueness of the equilibrium MF flow so these two policies must have the same result in terms of generated population distribution. 

    So we proved that $\tilde\mf^{\smf_0}_{n} = \hat\mf^{\smf_0}_{n}$ for all $n \ge 0$ and $(\tilde\pol^{\smf_0}_{n})_{n \ge 0}$ is an associated equilibrium policy.

\end{proof}

\section{On the Convergence of Master Fictitious Play}

In this section we study the evolution of the averaged MF flow generated by the Master Fictitious Play algorithm, see Alg.~\ref{algMEFP}. We then introduce a continuous time version of this algorithm and prove its convergence at a linear rate. 

\paragraph{On the mixture of policies. }

Given $\avgpoppol_K = \text{UNIFORM}(\poppol_1, \dots, \poppol_K)$ and given an initial $\smf_0$, we first compute an average population distribution composed of $K$ subpopulations where subpopulation $k$ uses $\poppol_k$ to react to the current average population. Formally, recall that we define:
$$
\begin{cases}
    \mf_{k,0}^{\smf_0} = \smf_0, \qquad k=1,\dots,K
    \\
    \bar\mf_{K,0}^{\smf_0} = \frac{1}{K} \sum_{k=1}^K \mf_{k,0}^{\smf_0}
\end{cases}
$$
and for $n \ge 0$,
$$
\begin{cases}
    \mf_{k,n+1}^{\smf_0} = \phi(\mf_{k,n}^{\smf_0}, \poppol_k(\cdot|\cdot,\bar\mf_{K,n}^{\smf_0})), \qquad k=1,\dots,K
    \\
    \bar\mf_{K,n+1}^{\smf_0} = \frac{1}{K} \sum_{k=1}^K \mf_{k,n+1}^{\smf_0}.
\end{cases}
$$
We recall that the notation $\mf_{k,n+1}^{\smf_0} = \phi(\mf_{k,n}^{\smf_0}, \poppol_k(\cdot|\cdot,\bar\mf_{K,n}^{\smf_0}))$ means:
\begin{equation}
    \label{eq:avg-pop-evol-K}
    \mf_{k,n+1}^{\smf_0}(x) = \phi(\mf_{k,n}^{\smf_0}, \poppol_k(\cdot|x,\bar\mf_{K,n}^{\smf_0}))
    = \sum_{x'} \mf_{k,n}^{\smf_0}(x') \sum_{a} \poppol_k(a|x',\bar\mf_{K,n}^{\smf_0}) p(x|x',a,\bar\mf_{K,n}^{\smf_0}) , \qquad x \in \states.
\end{equation}
Hence: for all $x \in \states$,
\begin{align}
    \bar\mf_{K,n+1}^{\smf_0}(x)
    &= \frac{1}{K} \sum_{k=1}^K \sum_{x'} \mf_{k,n}^{\smf_0}(x') \sum_{a} \poppol_k(a|x',\bar\mf_{K,n}^{\smf_0}) p(x|x',a,\bar\mf_{K,n}^{\smf_0}) 
    \label{eq:evol-mfbar-K}
    \\
    &=  \sum_{x'} \bar\mf_{K,n}^{\smf_0}(x')  \sum_{a} \underbrace{\left( \frac{1}{K} \sum_{k=1}^K \frac{\mf_{k,n}^{\smf_0}(x')}{\bar\mf_{K,n}^{\smf_0}(x')} \poppol_k(a|x',\bar\mf_{K,n}^{\smf_0}) \right)}_{=:\bar{\poppol}_{K,n}^{\smf_0}(a|x',\bar\mf_{K,n}^{\smf_0})} p(x|x',a,\bar\mf_{K,n}^{\smf_0}),
    \label{eq:evol-mfbar-K-polK}
\end{align}
where, in the last expression, the first sum over $\{x' \in\states: \bar\mf_{K,n}^{\smf_0}(x')>0 \}$.
So the evolution of the average population can be interpreted as the fact that all the agents use the policy $\bar{\poppol}_{K,n}^{\smf_0}(a|x',\bar\mf_{K,n}^{\smf_0})$ given by the terms between parentheses above. Note that this policy depends on $\smf_0$ and $n$.

We then consider the reward obtained by an infinitesimal player from the average population. This player belongs to subpopulation $k$ with probability $1/K$. So the reward can be expressed as:
$$
    \frac{1}{K} \sum_{k=1}^K J(\smf_0, \poppol_k; \bar\mf_{K}^{\smf_0}).
$$
We expect that when $K \to +\infty$ (\textit{i.e.}, we run more iterations of the Master Fictitious Play algorithm, see Alg.~\ref{algMEFP}), then this quantity converges to the one obtained by a typical player in the Nash equilibrium starting from $\smf_0$, \textit{i.e.}:
$$
    J(\smf_0, \hat\pol; \hat\mf^{\smf_0})
$$
where $\hat\mf^{\smf_0} = \Phi(\smf_0, \hat\pol)$ with $\hat\pol \in \argmax_{\pol} J(\smf_0, \pol; \hat\mf^{\smf_0})$.

Note that $\bar{\poppol}_{K,n}^{\smf_0}(a|x',\bar\mf_{K,n}^{\smf_0})$ takes $\bar\mf_{K,n}^{\smf_0}$ as an input. However, this dependence is superfluous because $\bar\mf_{K,n}^{\smf_0}$ can be derived from $\smf_0$ and $(\bar{\poppol}_{K,m}^{\smf_0}(a|x',\bar\mf_{K,m}^{\smf_0}))_{m \le n}$. Proceeding by induction, we can show that there exists $\bar{\pol}_{K}^{\smf_0} \in \Pol$ s.t.
$$
    \bar{\poppol}_{K,n}^{\smf_0}(a|x',\bar\mf_{K,n}^{\smf_0}) 
    = 
    \bar{\pol}_{K,n}^{\smf_0}(a|x')
$$

\paragraph{Continuous Time Master Fictitious Play.} 

We now describe the Continuous Time Master Fictitious Play (CTMFP) scheme in our setting. Here the iteration index $k \in\{1,2,3,\dots\}$ is replaced by a time $t$, which takes continuous values in $[1,+\infty)$. Intuitively, it corresponds to the limiting regime where the updates happen continuously. 

Based on~\eqref{eq:evol-mfbar-K}, we introduce the CTMFP mean-field flow defined for all $t\geq 1$ by: $\bar \mf_{t,n}^{\smf_0} = \mf_{t,n}^{\smf_0,\BR} = \smf_0$,  and for $n = 1,2,\dots,$
\begin{align}
    \bar \mf_{t,n}^{\smf_0} (x) = \frac{1}{t} \int \limits_{s=0}^t \mf_{s,n}^{\smf_0,\BR} (x) ds, \quad
    \textrm{ or in differential form: }
    \quad \frac{d}{dt} \bar \mf_{t,n}^{\smf_0} (x) = \frac{1}{t}\left(\mf_{t,n}^{\smf_0,\BR} (x)-\bar \mf_{t,n}^{\smf_0} (x)\right)\,,
    \label{eq:dyn-barmu-mf0}
\end{align}
where $\mf_{t,n}^{\smf_0,\BR}$ denotes the distribution induced by a best response policy $(\pol_{t,n}^{\smf_0,\BR})_{n \ge 0}$ against $\bar\mf_{t,n}^{\smf_0} (x)$.

As in~\eqref{eq:evol-mfbar-K-polK} for the discrete update case, the distribution $\bar\mf_{t,n}^{\smf_0}$ corresponds to the population distribution induced by the averaged policy $(\bar\pol_{t,n}^{\smf_0})_{n}$ defined as follows: for all $n = 1,2,\dots,$ and all $t \ge 1$:
\begin{align}
    &\;\bar\pol_{t,n}^{\smf_0}(a|x) \int \limits_{s=0}^t \mf_{s,n}^{\smf_0,\BR}(x) ds = \int \limits_{s=0}^t \mf_{s,n}^{\smf_0,\BR}(x) \pol_{s,n}^{\smf_0,\BR}(a|x) ds
    \label{eq:dyn-barmu-barpi-mf0}
    \\
    &\textrm{ or in differential form: } \bar\mf_{t,n}^{\smf_0}(x) \frac{d}{dt} \bar\pol_{t,n}^{\smf_0}(a|x)  = \frac{1}{t} \mf_{t,n}^{\smf_0,\BR}(x) [\pol_{t,n}^{\smf_0,\BR}(a|x) - \bar\pol_{t,n}^{\smf_0}(a|x)].
    \label{eq:dyn-barmu-barpi-mf0-diff}
\end{align}

The CTMFP process really starts from time $t=1$, but it is necessary to define what happens just before this starting time. For $t\in [0,1)$, we define $\bar\pol _{t<1}^{\smf_0} = (\bar\pol_{t<1,n}^{\smf_0})_{n} = (\pol_{t<1,n}^{\smf_0,\BR})_{n}$, where $\pol_{t<1}^{\smf_0,\BR}$ is constant and equal to an arbitrary policy. The induced distribution between time $0$ and $1$ is $\bar \mf_{t<1}^{\smf_0} = \mf_{t<1}^{\smf_0}=\mf^{\smf_0,\pol_{t<1}^{\smf_0}} = (\mf^{\smf_0,\pol_{t<1}^{\smf_0}}_n )_{n \ge 0}$.

\paragraph{Proof of convergence. }

We assume the transition $p$ is independent of the distribution: $x_{n+1} \sim p(.|x_n, a_n)$, and we assume the reward can be split as:
\begin{equation}
    \label{eq:reward-split}
    r(x,a,\smf) = r_A(x,a) + r_M(x,\smf).
\end{equation}
A useful property is the so-called monotonicity condition, introduced by~\citet{lasry2007mean}. 
\begin{definition}
    The MFG is said to be \emph{monotone} if: for all $\smf \neq \smf' \in \sMF$,
\begin{equation}
\label{eq:def-mono-mu-mup}
    \sum_{x} (\smf(x) - \smf'(x)) (r_M(x, \smf) - r_M(x, \smf')) < 0.
\end{equation}
\end{definition}
This condition intuitively means that the agent gets a lower reward if the population density is larger at its current state. 
Monotonicity implies that for every $\smf_0$, there exists at most one MF Nash equilibrium consistent with $\smf_0$; see~\citep{lasry2007mean}. This can be checked by considering the exploitability. 

Here, we are going to use the average exploitability as introduced in~\eqref{eq:avg-exploitability}:
\begin{equation*}
    \bar\cE_{\trainset}(\bar\pol_{t}) = \EE_{\smf_0 \sim \text{UNIFORM}(\trainset)} \big[ \bar\cE(\smf_0, \bar\pol_{t}^{\smf_0}) \big], 
\end{equation*}
where $\bar\pol_{t} = (\bar\pol_{t}^{\smf_0})_{\smf_0 \in \trainset}$ is the uniform distribution over past best responses $(\pol_{s}^{\smf_0,\BR})_{s \in [0,t], \smf_0 \in \trainset}$, and we define in the continuous-time setting:
$$
    \bar\cE(\smf_0, \bar\pol_{t}^{\smf_0}) = \max \limits_{\pol'} J(\smf_0, \pol'; \bar\mf^{\smf_0}_{t}) - \frac{1}{t}\int_{s=0}^t J(\smf_0, \pol_{t}^{\smf_0,\BR}; \bar\mf_{t}^{\smf_0}).
$$

\begin{theorem}[Theorem~\ref{thm:convergence-monotone-mf0} restated]
    Assume the reward is separable, the MFG is monotone, and the transition is independent of the population. Then, for every $\smf_0 \in \trainset$, $\bar\cE(\bar\pol_{t}) \in O(1/t)$.
\end{theorem}

\begin{proof}%

We follow the proof strategy of~\citet{perrin2020fictitious}, adapted to our setting. To alleviate the notation, we denote $\langle f,g\rangle_{_\actions}\, = \sum_{a \in \actions} f(a)g(a)$ for two functions $f,g$ defined on $\actions$, and similarly for $\langle \cdot,\cdot \rangle_{_\states}\,$. 
We also denote:
$
    r^{\spol}(x,\smf) = \langle \spol(\cdot|x), r(x,\cdot,\smf)\rangle_{_\actions}\,.
$

We first note that, by the structure of the reward function given in~\eqref{eq:reward-split},

$$
    \nabla_{\mu}r^{\pol_{t,n}^{\smf_0,\BR}}(x, \bar\mf_{t,n}^{\smf_0}) = \nabla_{\mu}r_M(x, \bar\mf_{t,n}^{\smf_0}) \textrm{ and } \nabla_{\mu}r^{\bar\pol_{t,n}^{\smf_0}}(x, \bar\mf_{t,n}^{\smf_0}) = \nabla_{\mu}r_M(x, \bar\mf_{t,n}^{\smf_0}).
$$

Moreover, using~\eqref{eq:dyn-barmu-barpi-mf0-diff} and~\eqref{eq:dyn-barmu-mf0} respectively, we have, for every $x \in \states$,
\begin{align*}
    - \langle\frac{d}{dt} \bar\pol_{t,n}^{\smf_0}(.|x), r(x,. ,\bar\mf_{t,n}^{\smf_0}) \rangle_{_\actions}\, \bar\mf_{t,n}^{\smf_0}(x)
    &=
    -\frac{1}{t}r^{\pol_{t,n}^{\smf_0,\BR}}(x,\bar\mf_{t,n}^{\smf_0})\mf_{t,n}^{\smf_0,\BR}(x)+\frac{1}{t} r^{\bar\pol_{t,n}^{\smf_0}}(x,\bar\mf_{t,n}^{\smf_0})\mf_{t,n}^{\smf_0,\BR}(x),
    \\
    - r^{\bar\pol_{t,n}^{\smf_0}}(x, \bar\mf_{t,n}^{\smf_0}) \frac{d}{dt} \bar\mf_{t,n}^{\smf_0}(x)
    &=
    \frac{1}{t}r^{\bar\pol_{t,n}^{\smf_0}}(x, \bar\mf_{t,n}^{\smf_0})\bar\mf_{t,n}^{\smf_0}(x) - \frac{1}{t}r^{\bar\pol_{t,n}^{\smf_0}}(x, \bar\mf_{t,n}^{\smf_0})\mf_{t,n}^{\smf_0,\BR}(x).
\end{align*}
Using the definition of exploitability together with the above remarks, we deduce:
\begin{align*}
    \frac{d}{dt} \bar\cE(\smf_0, \bar\pol_{t}^{\smf_0}) 
    &= \frac{d}{dt}\left[\max \limits_{\pol'} J(\smf_0, \pol'; \mf^{\bar\pol_{t}^{\smf_0},\smf_0}) - J(\smf_0, \bar\pol_{t}^{\smf_0}; \mf^{\bar\pol_{t}^{\smf_0},\smf_0})\right]
    \\
    &=\sum \limits_{n=0}^{+\infty} \, \gamma^n\,\sum \limits_{x \in \states} \Big[\langle\nabla_{\mu}r^{\pol_{t,n}^{\smf_0,\BR}}(x, \bar\mf_{t,n}^{\smf_0}), \frac{d}{dt} \bar\mf_{t,n}^{\smf_0}\rangle_{_\states}\,\mf_{t,n}^{\smf_0,\BR}(x) 
    \\ 
    &\qquad\qquad- \langle\nabla_{\mu} r^{\bar\pol_{t,n}^{\smf_0}}(x, \bar\mf_{t,n}^{\smf_0}), \frac{d}{dt} \bar\mf_{t,n}^{\smf_0}\rangle_{_\states}\, \bar\mf_{t,n}^{\smf_0}(x)
    \\
    &\qquad\qquad- \langle\frac{d}{dt} \bar\pol_{t,n}^{\smf_0}(\cdot|x), r(x, \cdot ,\bar\mf_{t,n}^{\smf_0})\rangle_{_\actions}\, \bar\mf_{t,n}^{\smf_0}(x) - r^{\bar\pol_{t,n}^{\smf_0}}(x, \bar\mf_{t,n}^{\smf_0}) \frac{d}{dt} \bar\mf_{t,n}^{\smf_0}(x)\Big]
    \\
    &= \sum \limits_{n=0}^{+\infty} \,\gamma^n\,\sum \limits_{x \in \states} \Big[t \langle \nabla_{\mu}r_M(x, \bar\mf_{t,n}^{\smf_0})), \frac{d}{dt} \bar\mf_{t,n}^{\smf_0}\rangle_{_\states}\, \frac{1}{t}\left(\mf_{t,n}^{\smf_0,\BR}(x)-\bar\mf_{t,n}^{\smf_0}(x)\right)\Big]
    \\
    &\qquad+ \sum \limits_{n=0}^{+\infty} \,\gamma^n\,  \sum \limits_{x \in \states}  \Big[\frac{1}{t}r^{\bar\pol_{t,n}^{\smf_0}}(x, \bar\mf_{t,n}^{\smf_0})\bar\mf_{t,n}^{\smf_0}(x) - \frac{1}{t}r^{\pol_{t,n}^{\smf_0,\BR}}(x,\bar\mf_{t,n}^{\smf_0})\mf_{t,n}^{\smf_0,\BR}(x)\Big]
    \\
    &=  
    - \frac{1}{t} \bar\cE(\smf_0, \bar\pol_{t}^{\smf_0}) + \sum \limits_{n=0}^{+\infty}  \,\gamma^n\, \sum \limits_{x \in \states} \Big[t \langle \nabla_{\mu}r_M(x, \bar\mf_{t,n}^{\smf_0}), \frac{d}{dt} \bar\mf_{t,n}^{\smf_0}\rangle_{_\states}\, \frac{d}{dt} \bar\mf_{t,n}^{\smf_0}(x)\Big],
\end{align*}
where the last term is non-positive. Indeed, the monotonicity condition~\eqref{eq:def-mono-mu-mup} implies that, for all $\tau \geq 0$, we have:
$$
    \sum \limits_{x \in \states} (\bar\mf_{t,n}^{\smf_0}(x) - \bar\mf_{t+\tau,n}^{\smf_0}(x))(r_M(x,\bar\mf_{t,n}^{\smf_0}) - r_M(x,\bar\mf_{t+\tau,n}^{\smf_0}))\leq 0.
$$
The result follows after dividing by $\tau^2$ and letting $\tau$ tend to $0$.

\end{proof}

\end{document}